%% file: main.tex
\documentclass[runningheads]{llncs}
\usepackage[T1]{fontenc}
\usepackage{graphicx}
\input{preamble}

\begin{document}
\title{Expediting Reinforcement Learning by Incorporating Knowledge About Temporal Causality in the Environment}
\titlerunning{Temporal Causality and Reinforcement Learning}
%
\author{Jan Corazza\inst{1,2} \and Hadi Partovi Aria\inst{3} \and Daniel Neider\inst{1,2} \and Zhe Xu\inst{3}\thanks{A peer-reviewed version appears in the Proceedings of the Third Conference on Causal Learning and Reasoning (PMLR 236:643-664, 2024)~\cite{tcrl}. Please cite the proceedings version. Source code: \url{https://github.com/corazza/tcrl}.}}

\authorrunning{J. Corazza et al.}

\institute{TU Dortmund University \and
  Research Center Trustworthy Data Science and Security\\
  \email{\{jan.corazza,daniel.neider\}@tu-dortmund.de}
  \and
  Arizona State University\\
  \email{\{hpartovi,xzhe1\}@asu.edu}}

\maketitle              
\begin{abstract}
  \input{abstract}

\end{abstract}
%
%
%

\input{sections/1_introduction}
\input{sections/2_problem_statement}
\input{sections/3_method}
\input{sections/4_case_studies}
\input{sections/5_conclusion_and_further_work}

\begin{credits}
  \subsubsection{\ackname} This work was supported in part by the National Science Foundation (NSF) under Grants CNS 2304863 and CNS 2339774, and in part by the Office of Naval Research (ONR) under Grant N00014-23-1-2505. Additionally, this work has been financially supported by the Research Center Trustworthy Data Science and Security~\footnote{\url{https://rc-trust.ai}}, one of the Research Alliance centers within the UA Ruhr~\footnote{\url{https://uaruhr.de}}.

\end{credits}
%
%
%
\bibliographystyle{splncs04}
\bibliography{references}
%

\clearpage
\appendix

Section~\ref{sec:product} contains the algorithm for the \texttt{computeProduct} function used in Algorithm~\ref{alg:causal}.
Section~\ref{sec:theory} contains the proof of Theorem~\ref{thm:convergence}.
In Section~\ref{sec:4-door-case-study}, the PRM related to the four-door task is depicted.

\input{appendix_sections/1_product}
\input{appendix_sections/2_proofs}
\input{appendix_sections/3_additional_results}

\end{document}

%% file: preamble.tex
\usepackage[utf8]{inputenc} 
\usepackage[T1]{fontenc}    
\usepackage{hyperref}       
\usepackage{booktabs}       
\usepackage{amsfonts}       
\usepackage{nicefrac}       
\usepackage{microtype}      
\usepackage{xcolor}         
\usepackage{dsfont}
\usepackage{xspace}
\usepackage{bbold}
\usepackage{tikz}
\usepackage{subfigure}
\usepackage{pgfplots}
\usepackage{pgfplotstable}
\usepackage{algpseudocode}
\usepackage{fontawesome5}
\usepackage{times}
\usepackage{amsmath}        
\usepackage{amssymb}       
\usepackage{graphicx}      
\usepackage{url}           
\usepackage{algorithm2e}   

\LinesNumbered{}

\usetikzlibrary{positioning, quotes, calc}
\usetikzlibrary{tikzmark}
\usetikzlibrary{arrows, automata}
\usetikzlibrary{shapes.geometric}

\pgfplotsset{compat=1.18}

\input{macros}



%% file: macros.tex
\newcommand{\jan}[1]{#1}
\newcommand{\hadi}[1]{#1}

\newcommand{\expect}{\ensuremath{\mathds{E}}}
\newcommand{\mathset}[1]{\{ #1 \}}
\newcommand{\indicator}[1]{\ensuremath{\mathbb{1}_{\mathset{#1}}}}
\newcommand{\reals}{\mathbb{R}}

\newcommand{\fsm}[1]{\mathsf{#1}}
\newcommand{\dfa}[1]{\mathsf{#1}}
\newcommand{\languageOf}{\ensuremath{\mathcal{L}}}
\newcommand{\dfaStates}{Q}
\newcommand{\dfaCommonState}{q}
\newcommand{\dfaInitialState}{\dfaCommonState_{I}}
\newcommand{\dfaTransitions}{\delta}
\newcommand{\dfaAcceptingStates}{F}
\newcommand{\dfaRejectingSinks}{\dfaStates_{\text{r.s.}}}

\newcommand{\mdp}[1]{#1}
\newcommand{\mdpMatrixTransition}{\ensuremath{\mathcal{P}}}
\newcommand{\mdpMatrixReward}{\ensuremath{\mathcal{R}}}
\newcommand{\policy}{\ensuremath{\pi}}
\newcommand{\mdpStates}{S}
\newcommand{\mdpCommonState}{s}
\newcommand{\mdpActions}{A}
\newcommand{\mdpCommonAction}{a}
\newcommand{\mdpReward}{R}
\newcommand{\mdpCommonReward}{r}
\newcommand{\mdpDynamics}{p}
\newcommand{\mdpDiscount}{\gamma}
\newcommand{\mdpTrajectory}[1]{\ensuremath{
\mdpCommonState_0, \mdpCommonAction_0, \mdpCommonState_1, \ldots, \mdpCommonAction_{#1-1}, \mdpCommonState_{#1}
}}

\newcommand{\stateValueFunction}{v}
\newcommand{\optimalStateValueFunction}{\stateValueFunction^{\star}}
\newcommand{\qFunction}{q}
\newcommand{\optimalQFunction}{\qFunction^{\star}}

\newcommand{\atomic}{\ensuremath{\text{AP}}}
\newcommand{\propAlphabet}{\ensuremath{2^\atomic}}

\newcommand{\propInput}{\ell}
\newcommand{\propInputSeq}[1]{\ensuremath{\propInput_0 \propInput_1 \cdots \propInput_{#1}}}
\newcommand{\labelingFunction}{L}

\newcommand{\startLabel}{\faRobot{}}

\newcommand{\prm}[1]{\fsm{#1}}
\newcommand{\prmStates}{U}
\newcommand{\prmTerminals}{F}
\newcommand{\prmCommonState}{u}
\newcommand{\prmInitState}{\prmCommonState_{I}}
\newcommand{\prmInputAlphabet}{\propAlphabet}
\newcommand{\dfaInputAlphabet}{\Sigma}
\newcommand{\prmOutputAlphabet}{\Gamma}
\newcommand{\prmTransitions}{\tau}
\newcommand{\prmOutput}{\sigma}

\newcommand{\causalDiagram}{\mathcal{C}}

\newcommand{\ltlf}{\ensuremath{\text{LTL}_f}}

\newcommand{\nextOp}{\ensuremath{\textbf{X}}}
\newcommand{\untilOp}{\ensuremath{\textbf{U}}}
\newcommand{\weakUntilOp}{\ensuremath{\textbf{W}}}

\newcommand{\globallyOp}{\ensuremath{\textbf{G}}}

\newtheorem{mylemma}{Lemma}

\newtheorem{mytheorem}{Theorem}

\newcommand*{\highlightyellow}[1]{%
  \tikz[baseline=(X.base)] \node[rectangle, fill=yellow!15, rounded corners, inner sep=0.5mm] (X) {#1};%
}

\newcommand*{\highlightblue}[1]{%
  \tikz[baseline=(X.base)] \node[rectangle, fill=blue!15, rounded corners, inner sep=0.5mm] (X) {#1};%
}


%% file: abstract.tex
Reinforcement learning (RL) algorithms struggle with learning optimal policies for tasks
where reward feedback is sparse and depends on a complex sequence of events in the environment.
Probabilistic reward machines (PRMs) are finite-state formalisms that can capture temporal dependencies in the reward signal, along with nondeterministic task outcomes.
While special RL algorithms can exploit this finite-state structure to expedite learning,
PRMs remain difficult to modify and design by hand.
This hinders the already difficult tasks of utilizing high-level causal knowledge about the environment,
and transferring the reward formalism into a new domain with a different causal structure.
This paper proposes a novel method to incorporate causal information in the form of Temporal Logic-based Causal Diagrams into the reward formalism,
thereby expediting policy learning and
aiding the transfer of task specifications to new environments.
Furthermore, we provide a theoretical result about convergence to optimal policy for our method,
and demonstrate its strengths empirically.

%% file: sections/1_introduction.tex
\section{Introduction}
\hadi{
    Reinforcement Learning (RL) has emerged as the forefront method in providing a robust and general framework for intelligent, autonomous decision-making and learning within complex environments.
    One of the biggest challenges in Reinforcement Learning is integrating high-level, causal knowledge into the learning process.
    Causal reasoning may come naturally to humans, assisting them in navigating the world by making decisions based on more than just observed outcomes; it involves an understanding of how those outcomes come about.
    This is in contrast to traditional RL techniques, which often lack the ability to capture temporal cause-effect relationships and hence offer inefficient learning and decision-making.
    For instance, the knowledge of the likely consequences of actions in terms of future states and rewards can dramatically reduce the amount of exploration that needs to take place to learn effective policies.
    This problem is most evident in settings with long-term consequences, which calls for RL methods that could incorporate causal knowledge directly into their decision process.
}

In RL the interaction between the agent and the environment happens step by step.
Starting in state $\mdpCommonState$ the agent chooses an action $\mdpCommonAction$ with probability $\policy(\mdpCommonAction \mid \mdpCommonState)$ (the policy),
and the environment transitions into a new state $\mdpCommonState'$ and gives a reward $\mdpCommonReward$.
This interaction is formalized in the concept of an MDP,
a tuple $\mdp{M} = (\mdpStates, \mdpActions, \mdpReward, \mdpDynamics, \mdpDiscount)$
where $\mdpStates$ is the set of states,
$\mdpActions$ the set of actions available to the agent,
$\mdpReward : (\mdpStates \times \mdpActions)^\star \times \mdpStates \to \reals$
the reward function mapping trajectories in the MDP to rewards,
$\mdpDynamics(\mdpCommonState' \mid \mdpCommonState, \mdpCommonAction)$
a probabilistic transition function,
and $\mdpDiscount \in (0, 1)$ the discount factor.
The agent's goal is to maximize the expected discounted return,
$\max_{\policy}\expect_{\policy}[\sum_{i=0}^{\infty} \mdpDiscount^i \mdpCommonReward_i]$.
A labeling function
$\labelingFunction(\mdpCommonState, \mdpCommonAction, \mdpCommonState')$
can be provided to attach descriptive propositional variables to transitions in the MDP.
An MDP together with a labeling function is called a labeled MDP.

Although MDPs can have a large number of states and a complex transition function,
one often has access to high-level causal knowledge of the environment.
Figure~\ref{fig:gridworld-coffee-soda} illustrates this point on a small example MDP.
To complete the task, the agent must choose to bring either coffee or a soda to the office.
The high-level knowledge one may supply
is that any path from the soda to the office is later blocked by a flower pot,
which the agent must avoid.
This is due to walls and a one-way door, which constrain the agent's movement.
Although special RL algorithms can find the optimal policy for this task,
they will not take these temporal-causal constraints into account,
and will explore the environment in an inefficient manner.
Unfortunately, employing high-level knowledge about causality has shown to be a difficult task, as the current causal RL approaches (e.g., \cite{Zhang2020a,YangyiLu2021,Wang2021,Bareinboim2015,LeeSanghack2018,Mesnard2021,AndrewForney2017a,Li2021CausalRL}) mostly do not take into account the \textit{temporal} aspect of the causal knowledge.
This paper aims to address this issue by proposing a novel method that incorporates knowledge about causality directly into the reward function.
On the other hand, den Hengst et al.~\cite{denhengst2022planning} propose an approach for safe RL that incorporates symbolic reasoning and a temporal domain. However, their primary focus is on ensuring safety rather than expediting the learning process, distinguishing their work from ours, which specifically targets efficient learning by leveraging causality.

\input{figures/mdp-coffee-soda}

\subsection{Probabilistic Reward Machines}

Common RL algorithms such as Q-learning struggle with tasks where rewards are sparse and depend on a complex sequence of actions that the agent must perform in a specific order.
Reward machines, introduced by Icarte et al.~\cite{DBLP:journals/corr/abs-2010-03950}, are a finite-state formalism that can capture the reward function in such cases.
Q-learning for Reward Machines (QRM), as proposed by Icarte et al.~\cite{DBLP:journals/corr/abs-2010-03950}, can exploit this reward structure to expedite learning the optimal policy.
The QRM algorithm employs reward machines to significantly enhance the efficiency of problem-solving processes.
This algorithm applies an off-policy Q-learning strategy to the reward machine by decomposing it into relevant components at the same time, thereby facilitating the simultaneous learning of each distinct subpolicy.
This methodological approach has been empirically validated, demonstrating the algorithm's capability to converge towards an optimal policy in tabular case.
Velasquez et al.~\cite{DBLP:journals/corr/abs-2107-04633} introduced a more general variant of reward machines called \emph{probabilistic reward machines} (PRMs).
PRMs use a nondeterministic transition function that can capture uncertainty in task outcomes.
In the example from Figure~\ref{fig:mdp-coffee-soda},
uncertainty comes from the fact that the coffee machine may malfunction.
Definition~\ref{def:prm} formalizes this notion of a
finite-state representation of a temporally extended task with probabilistic outcomes.

\begin{definition}[Probabilistic Reward Machine (PRM)]
    \label{def:prm}
    A PRM
    $\prm{A} = (\prmStates, \prmInitState, \prmInputAlphabet, \prmOutputAlphabet, \prmTransitions,\allowbreak \prmOutput, \prmTerminals)$ is a tuple where
    $\prmStates$ is a finite set of states
    with a distinguished initial state $\prmInitState \in \prmStates$,
    $\atomic$ is a set of atomic propositions and $\prmInputAlphabet$ is the set of labels,
    $\prmOutputAlphabet \subset \reals$ is a finite set of rewards,
    $\prmTransitions : (\prmStates \times \prmInputAlphabet \times \prmStates) \to [0, 1]$ is a probabilistic transition function,
    $\prmOutput : (\prmStates \times \prmInputAlphabet \times \prmStates ) \to \prmOutputAlphabet$ is a function mapping each transition to a reward in $\prmOutputAlphabet$, and
    $\prmTerminals \subseteq \prmStates$ is a finite set of terminal states that signal the end of the interaction.
\end{definition}

The agent-environment interaction generates a trajectory $\mdpTrajectory{n}$
and the corresponding label sequence $\propInputSeq{n-1}$,
where $\labelingFunction(\mdpCommonState_i, \mdpCommonAction_i, \mdpCommonState_{i+1}) = \propInput_i$
for all $i = 0, \ldots, n-1$.
The state $\mdpCommonState_0$ may be a unique initial state, or drawn from an initial distribution.
After reading a label $\propInput$ in state $\prmCommonState$,
the PRM executes a nondeterministic transition into a new state $\prmCommonState'$
with probability $\prmTransitions(\prmCommonState, \propInput, \prmCommonState')$,
and the agent receives a reward
$\mdpCommonReward = \prmOutput(\prmCommonState, \propInput, \prmCommonState')$.
A run of a PRM $\prm{A}$ on a label sequence $\propInputSeq{n-1}$ is a sequence
$\prmCommonState_0, \mdpCommonReward_0, \prmCommonState_1, \ldots, \mdpCommonReward_{n-1}, \prmCommonState_n$
where $\prmCommonState_0 = \prmInitState$,
and for all $i = 0, \ldots, n-1$,
$\prmTransitions(\prmCommonState_{i}, \propInput_i, \prmCommonState_{i+1}) > 0$ and
$\prmOutput(\prmCommonState_i, \propInput_i, \prmCommonState_{i+1}) = \mdpCommonReward_i$.

\subsection{Temporal Logic-based Causal Diagrams}
\label{sec:tlcds}

Linear temporal logic over finite sequences ($\ltlf$) is a formal reasoning system
that can capture causal and temporal properties of label sequences and labeled MDPs.
Aside from Boolean operators like $\lnot$ and $\lor$, $\ltlf$ introduces temporal operators such as $\globallyOp \psi$ (true if and only if $\psi$ holds for every element in the sequence), $\nextOp \psi$ (true iff. $\psi$ holds for the next element of the sequence), and $\psi \untilOp \varphi$ (true iff. $\psi$ holds until $\varphi$ becomes true, and $\varphi$ is true in some element of the sequence).
We also rely on the weak until operator $\psi \weakUntilOp \varphi$
(true iff. $\psi$ holds until $\varphi$ becomes true, but $\varphi$ is not required to become true).

In order to encode knowledge about causality in the underlying MDP, we rely on Temporal Logic-based Causal Diagrams (TL-CDs) introduced in Paliwal et al.~\cite{10.1007/978-3-031-40837-3_8}.
TL-CDs are a special notation that expresses the causal relationship between formulas in $\ltlf$.
The first conjunct induced by the TL-CD in Figure~\ref{fig:cd-soda},
$\globallyOp(\texttt{s} \rightarrow \lnot \texttt{o} \weakUntilOp \texttt{f})$,
means that if the agent observes $\texttt{s}$ (soda) in any step, then it will not observe $\texttt{o}$ (the office) before it observes $\texttt{f}$ (the flower pot).
This part of the TL-CD encodes knowledge that soda may only be reached via a one-way door,
and the only other exit towards the office will be blocked by the flower pot.

\input{figures/cd}

Formally, a TL-CD is a directed graph whose nodes are labeled with $\ltlf$ formulas.
For a TL-CD $\causalDiagram$ one may construct an equivalent $\ltlf$ formula
$\varphi^\causalDiagram$ through Equation~\ref{eqn:tlcd},
where $\varphi \blacktriangleright \psi$ iterates over edges that connect formulas $\varphi$ and $\psi$ in the TL-CD.

\begin{equation}
    \varphi^\causalDiagram = \bigwedge\limits_{\varphi \blacktriangleright \psi}
    \globallyOp(\varphi \rightarrow \psi)
    \label{eqn:tlcd}
\end{equation}

If $\varphi^\causalDiagram$ is true for a label sequence $\propInput$,
we will write $\propInput \models \varphi^\causalDiagram$.
A label sequence $\propInput = \propInputSeq{n-1}$ is attainable in an MDP
$\mdp{M} = (\mdpStates, \mdpActions, \mdpReward, \mdpDynamics, \mdpDiscount)$
if there exists a trajectory $\mdpTrajectory{n}$ in $\mdp{M}$ such that
$\labelingFunction(\mdpCommonState_i, \mdpCommonAction_i, \mdpCommonState_{i+1}) = \propInput_i$ and
$\mdpDynamics(\mdpCommonState_i, \mdpCommonAction_i, \mdpCommonState_{i+1}) > 0$
for all $i = 0, 1, \ldots, n-1$.
We will say that a TL-CD $\causalDiagram$ holds for an MDP $\mdp{M}$
if for every label sequence $\propInput$ attainable in $\mdp{M}$,
we have $\propInput \models \varphi^\causalDiagram$.
\jan{In order to simplify working with TL-CDs, we leverage the notion of deterministic finite automata (DFAs).
    We formalize this notion in Definition~\ref{def:dfa}.}

\begin{definition}[Deterministic Finite Automaton (DFA)]
    A DFA is a tuple $\dfa{C} = (\dfaStates, \dfaInitialState, \dfaInputAlphabet,\allowbreak \dfaTransitions, \dfaAcceptingStates)$
    consisting of a finite set of states $\dfaStates$ with an initial state $\dfaInitialState$, input alphabet $\dfaInputAlphabet$, deterministic transition function $\dfaTransitions : \dfaStates \times \dfaInputAlphabet \to \dfaStates$,
    and a finite set of accepting states $\dfaAcceptingStates \subseteq \dfaStates$.
    \label{def:dfa}
\end{definition}

If the run of the DFA $\dfa{C}$ on an input string $\propInput$ ends in an accepting state $\dfaCommonState \in \dfaAcceptingStates$, we will write $\propInput \in \languageOf(\fsm{C})$.
Every TL-CD $\causalDiagram$ can be converted into an equivalent
DFA $\fsm{C}$,
in the sense that for every $\propInput$,
we have $\propInput \in \languageOf(\fsm{C}) \iff \propInput \models \varphi^\causalDiagram$.
We will refer to $\dfa{C}$ as the \emph{causal DFA}.

%% file: figures/mdp-coffee-soda.tex
\begin{figure}[t]
    \centering

    \subfigure[A labeled \texttt{5x5} Gridworld with coffee (\texttt{c}), soda (\texttt{s}), an office (\texttt{o}), and a flower pot (\texttt{f}).
        The agent can move in the four cardinal directions,
        and starts in the cell labeled \startLabel{}.
        Other shaded cells are impassable walls.
        One-way doors are represented by the upwards arrow.
        The flower pot acts as a sink state. ]{%
        \begin{minipage}{0.45\textwidth}
            \centering
            \begin{tikzpicture}[scale=0.8]
                \draw (0,0) grid (5,5);
                \fill[green!10] (2.1,0.1) rectangle (2.9,0.9);
                \fill[red!10] (2.1-1,1.1) rectangle (2.9-1,1.9);
                \fill[red!10] (2.1,1.1) rectangle (2.9,1.9);
                \fill[red!10] (2.1,1.1+1) rectangle (2.9,1.9+1);
                \fill[red!10] (2.1,1.1+2) rectangle (2.9,1.9+2);
                \node at (2.5, 4.5) {\tikz[baseline=(char.base)]{\node[line width=1pt,blue,draw,circle,inner sep=2.8pt, text=black] (char) {\texttt{f}};}};
                \node at (2.5,0.5) {\startLabel{}};
                \node at (4.5,4.5) {\texttt{o}};
                \node at (4.5-1,4.5-2) {\texttt{c}};
                \node at (1.5,3.5) {\texttt{s}};
                \draw[->, blue, thick, line width=1pt] (0.5,1.3) -- (0.5,1.7);
            \end{tikzpicture}
            \label{fig:gridworld-coffee-soda}
        \end{minipage}
    }
    \hfill
    \subfigure[A PRM for the task in Figure~\ref{fig:gridworld-coffee-soda} (left).
        Transitions are labeled with propositional formulas and reward outputs.
        Only transition probabilities different from $1$ are shown.
        State $q_4$ is a terminal state which ends the task.
        \jan{Formulas on transitions from $q_0$ to $q_1$ and $q_2$ ($\texttt{c}\land \lnot\texttt{s}$) contain $\lnot\texttt{s}$
            in order to disambiguate the transition function in state $q_0$ on the input $\mathset{\texttt{c}, \texttt{s}}$.
            The same could be achieved by using the formula $\lnot\texttt{c}\land \texttt{s}$ for the transition from $q_0$ to $q_3$,
            and just $\texttt{c}$ for transitions into $q_1$, $q_2$.}]{%
        \begin{minipage}{0.45\textwidth}
            \centering
            \begin{tikzpicture}[->,>=stealth',shorten >=1pt,auto,node distance=3cm,semithick, scale=0.8]
                \node[state, initial, initial text=] (q0) at (-2, 2) {$q_0$};
                \node[state] (q1) at (2, 2) {$q_1$};
                \node[state] (q2) at (0, 0) {$q_2$};
                \node[state] (q3) at (-2, -2) {$q_3$};
                \node[state, accepting] (q4) at (2, -2) {$q_4$};
                \path (q0) edge[bend left=10] node[sloped]{$\texttt{\highlightblue{c}}\land \lnot\texttt{s}$, $0$} node[sloped, swap]{\highlightyellow{$0.9$}} (q1);
                \path (q0) edge[] node[sloped]{$\texttt{\highlightblue{c}}\land \lnot\texttt{s}$, $0$} node[sloped, swap]{\highlightyellow{$0.1$}} (q2);
                \path (q0) edge[bend right=10] node[sloped]{$\texttt{\highlightblue{s}}$, $0$} (q3);
                \path (q1) edge[bend left=10] node[sloped]{$\texttt{\highlightblue{o}}$, $1$} (q4);
                \path (q2) edge[] node[sloped]{$\texttt{\highlightblue{o}}$, $0.1$} (q4);
                \path (q3) edge[bend right=10] node[sloped]{$\texttt{\highlightblue{o}}$, $1$} (q4);
                \path (q0) edge[loop above] node[left]{$\lnot(\texttt{c} \lor \texttt{s})$, $0$} (q1);
                \path (q1) edge[loop above] node[left]{$\lnot\texttt{o}$, $0$} (q1);
                \path (q2) edge[loop below] node[left]{$\lnot\texttt{o}$, $0$} (q2);
                \path (q3) edge[loop left] node[left]{$\lnot\texttt{o}$, $0$} (q3);
            \end{tikzpicture}
            \label{fig:prm-coffee-soda}
        \end{minipage}
    }
    \caption{An MDP (left) and a PRM (right) that captures the task of bringing either coffee or soda to the office.
        The coffee machine has a probability of $10\%$ to malfunction and produce bad coffee, leading to a reduced reward of $0.1$ instead of $1$.
        Bringing soda to the office results in a reward of $1$ deterministically.
        \jan{An example input for the PRM is
            $\mathset{\texttt{c, s}}, \emptyset, \mathset{\texttt{o, c}}$ (a sequence of three labels),
            which will induce the run $q_0 \mapsto q_3 \mapsto q_3 \mapsto q_4$ with a reward of $1$.
            It is important to note that inputs for PRMs are sets of descriptive propositional variables that are true in a given step,
            hence why a single label such as $\mathset{\texttt{c, s}}$ can include multiple (or $0$) variables.}}
    \label{fig:mdp-coffee-soda}
\end{figure}
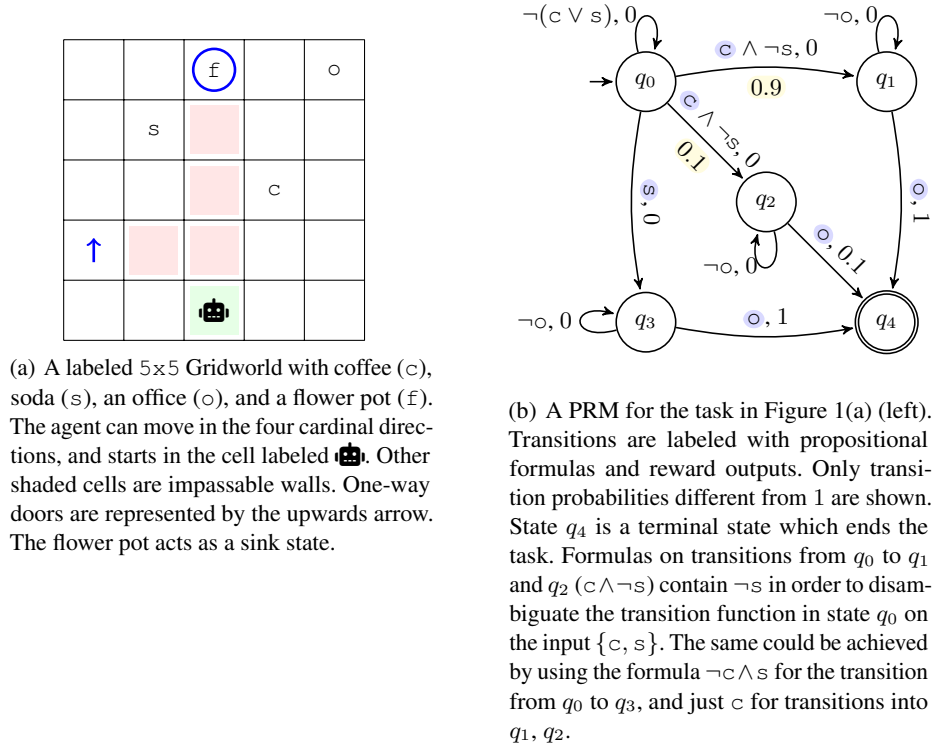

%% file: figures/cd.tex
\begin{figure}
    \centering
    \subfigure[$\globallyOp(\texttt{s} \rightarrow \lnot \texttt{o} \weakUntilOp \texttt{f}) \land
      \globallyOp(\texttt{f} \rightarrow \globallyOp \lnot \texttt{o})$ ]{%
    \begin{minipage}{0.45\textwidth}
    \centering
    \begin{tikzpicture}
      \node[draw, ultra thick, circle] (p1) at (0, 0) {$\texttt{s}$};
      \node[draw, ultra thick, circle] (p2) at (2, 0) {$\lnot \texttt{o} \weakUntilOp \texttt{f}$};
      \node[draw, ultra thick, circle] (p3) at (0, -1.5) {$\texttt{f}$};
      \node[draw, ultra thick, circle] (p4) at (2, -1.5) {$\globallyOp \lnot \texttt{o}$};
      \draw[-latex] (p1) -- (p2) node[midway, above] {};
      \draw[-latex] (p3) -- (p4) node[midway, above] {};
    \end{tikzpicture}
    \label{fig:cd-soda}
    \end{minipage}
    }
    \hfill
    \subfigure[{$\globallyOp(\texttt{a} \rightarrow \globallyOp \lnot \texttt{b})$}]{%
    \begin{minipage}{0.45\textwidth}
    \centering
    \begin{tikzpicture}
      \node[draw, ultra thick, circle] (p) at (0, 0) {$\texttt{a}$};
      \node[draw, ultra thick, circle] (q) at (2, 0) {$\globallyOp \lnot \texttt{b}$};
      \draw[-latex] (p) -- (q) node[midway, above] {};
    \end{tikzpicture}
    \label{fig:cd-collection}
    \end{minipage}
    }
    \caption{Figure~\ref{fig:cd-soda} (left) is the TL-CD which captures relevant causal information in the environment from Figure~\ref{fig:gridworld-coffee-soda}. Figure~\ref{fig:cd-collection} (right) is a TL-CD that holds for the case study in Figure~\ref{fig:mdp-collection_second_case}.}
    \label{fig:mdp-collection}
\end{figure}
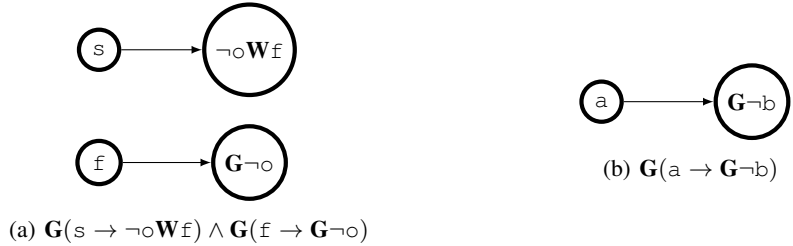

%% file: sections/2_problem_statement.tex
\section{Problem statement}

One may use QRM
to find the optimal policy for the task in Figure~\ref{fig:mdp-coffee-soda}.
However, the PRM in Figure~\ref{fig:prm-coffee-soda} does not take flower pots and one-way doors into account.
Because the agent does not know that knocking over flower pots is forbidden
or that choosing soda causes him to enter a room blocked by a flower pot,
it will waste time exploring those fruitless trajectories.
As PRMs are in essence task specifications,
and one may also wish to transfer them into a new environment while preserving the overall goal.
In both cases, high-level insights about causality,
especially its temporal aspects,
could prove helpful by reflecting the dynamics of the MDP in condensed form.

Unfortunately, incorporating knowledge about temporal causality into the reward function remains a difficult and error-prone manual task.
In PRMs, this would necessitate adding new states and reasoning about a different, more complicated transition function.
Some methods such as JIRP, proposed by Xu et al.~\cite{DBLP:journals/corr/abs-1909-05912} and SRMI, proposed by Corazza et al.~\cite{Corazza_Gavran_Neider_2022} assume that a suitable but unknown representation of the reward function exists,
and attempt to recover it from interaction traces.
This work proposes an alternative method that leverages TL-CDs in order to automate the process of incorporating knowledge about causality into PRMs.
More formally, the problem can be stated as follows.
Given a TL-CD $\causalDiagram$ which holds for an MDP $\mdp{M}$
and a PRM $\prm{A}$,
produce a PRM $\prm{B}$ that induces the same optimal policy as $\prm{A}$,
but utilizes causal information in $\causalDiagram$ to expedite learning.

%% file: sections/3_method.tex
\section{Method}
\label{sec:method}

This section introduces a novel approach for incorporating causal information from a TL-CD into the reinforcement learning process via a PRM. By creating a product of a TL-CD, represented as a causal DFA, with a PRM, we aim to enhance the RL agent's ability to efficiently learn and make decisions in complex environments.

We first consider the equivalent causal DFA for a given TL-CD.
As explained in Section~\ref{sec:tlcds}, the equivalent causal DFA captures the same semantics as the given TL-CD.
While TL-CDs are an intuitive notational tool, DFAs are easier to work with computationally.
The causal DFA for the TL-CD in Figure~\ref{fig:cd-soda} is shown in Figure~\ref{fig:causal-dfa}
(in two parts for convenience).
State $u_3$ is a \emph{sink} state, meaning that any run of the DFA which enters $u_3$ will never leave it.
It is also a \emph{rejecting} state.
Taken together, this means that any label sequence for which the causal DFA enters $u_3$ is not the prefix of an attainable sequence in an MDP $\mdp{M}$ if we assume that the TL-CD holds for $\mdp{M}$.

\input{figures/causal-dfa}

In other words,
when a causal DFA run reaches a rejecting sink state
on some prefix of input labels,
then the entire input label sequence is unattainable.
The reason is that there is no suffix of labels which can
cause the run to transition into an accepting state.
From now on we will implicitly consider causal DFAs to have at most one rejecting sink state,
and that an accepting state is reachable from all other states.
This can be achieved by minimization~\cite{sipser13}.

We propose to incorporate causal information from a TL-CD $\causalDiagram$ into a PRM $\prm{A}$
by computing state values in a new PRM
$\prm{B_1} = \causalDiagram \times \prm{A}$,
which is a product of the TL-CD (represented by the causal DFA $\dfa{C}$) and $\prm{A}$.
The product PRM $\prm{B_1}$ synchronizes the runs of the original PRM $\prm{A}$
and the causal DFA $\dfa{C}$.
$\prm{B_1}$ mirrors the output of $\prm{A}$,
except when $\dfa{C}$ transitions into a rejecting sink state.
Then the output of $\prm{B_1}$ is set to a minimal value $m$
that is lesser than any possible immediate reward and resulting future gain,
and will remain there for the rest of the run as $\dfa{C}$ will not leave the sink state.
We also compute state values in a ``pessimistic'' PRM $\prm{B_2} = \causalDiagram \times \prm{(-A)}$ in order to uncover temporal-causal information about worst-case reward outcomes.
While $\prm{B_1}$ outputs the same rewards as $\prm{A}$,
$\prm{B_2}$ negates outputs of $\prm{A}$
(but also gives minimal outputs $m$ for transitions into rejecting sink states).
Because of the minimal reward output $m$, value iteration in either $\prm{B_1}$ or $\prm{B_2}$ will disregard transitions that lead $\dfa{C}$ into a rejecting sink state, as explained under Figure~\ref{fig:product-fragment}.
Due to negating reward outputs, label sequences that maximize return in $\prm{B_2}$, minimize the return in $\prm{B_1}$.
We combine state value information from $\prm{B_1}$ and $\prm{B_2}$ into a final PRM $\prm{B}$.
To obtain $\prm{B}$, we start from $\prm{B_1}$, and add all states $\prmCommonState \in \prmStates^{\prm{B_1}}$ that have $0$ value in both the machine $\prm{B_1}$ and $\prm{B_2}$
($\optimalStateValueFunction_{\prm{B_1}}(\prmCommonState) = \optimalStateValueFunction_{\prm{B_2}}(\prmCommonState) = 0$)
into the set of terminal states $\prmTerminals^{\prm{B_1}}$.
Such states have the property that no matter the policy,
the future return is constrained with $0$ from above and below
(and thus, the choice of actions is of no consequence).
The product $\causalDiagram \times \prm{A}$ is formalized in Definition~\ref{def:product}.
We define the value of a PRM state $\prmCommonState$ via the Bellman optimality equation~\ref{eqn:state-value}, where $\mdpDiscount$ matches the discount factor in the MDP.
\begin{equation}
    \optimalStateValueFunction(\prmCommonState) =
    \max_{\propInput \in \prmInputAlphabet}
    \sum_{\prmCommonState' \in \prmStates}
    \prmTransitions\left(\prmCommonState, \propInput, \prmCommonState'\right) \cdot
    \left(\prmOutput\left(\prmCommonState, \propInput, \prmCommonState'\right) + \mdpDiscount \optimalStateValueFunction\left(\prmCommonState'\right)\right)
    \label{eqn:state-value}
\end{equation}

As Equation~\ref{eqn:state-value} is an optimality equation,
$\optimalStateValueFunction(\prmCommonState)$ is the expected return of a PRM run starting in $\prmCommonState$ and following the most optimistic label sequence
(which may or may not be attainable in the MDP).
We define the minimal reward output $m$ as $m = -1 - \max_{r \in \prmOutputAlphabet^\prm{A}}{|r|} - \max_{\prmCommonState \in \prmStates^\prm{A}}{\optimalStateValueFunction(\prmCommonState)}$.
While it may be simpler to use $m=-\infty$, we compute a concrete bound in order to better communicate how our method makes use of state value information. In brief, the formula for $m$ is inspired by the Bellman optimality operator used in value iteration. The terms can be explained in the following way. First, $- \max_{r \in \prmOutputAlphabet^\prm{A}}{|r|}$ ensures that the reward is lower than any other immediate reward in the original PRM. Second, $- \max_{\prmCommonState \in \prmStates^\prm{A}}{\optimalStateValueFunction(\prmCommonState)}$ ensures that the reward is lower than any possible future gain starting from a state in the original PRM. Taken together, these two terms ensure that transitions that correspond to rejecting sink states do not contribute to state values.

\begin{definition}[PRM \& TL-CD product]
    Let $\mdp{M} = (\mdpStates, \mdpActions, \mdpReward, \mdpDynamics, \mdpDiscount)$ be an MDP
    where the reward function $\mdpReward : (\prmInputAlphabet)^\star \to \prmOutputAlphabet$ is given by the PRM
    $\prm{A} = (\prmStates^\prm{A}, \prmInitState^\prm{A}, \prmInputAlphabet, \prmOutputAlphabet^\prm{A}, \prmTransitions^\prm{A}, \prmOutput^\prm{A},\allowbreak \prmTerminals^\prm{A})$,
    $\causalDiagram$ a TL-CD that holds for $\mdp{M}$, and
    $\dfa{C} = (\dfaStates, \dfaInitialState, \prmInputAlphabet, \dfaTransitions, \dfaAcceptingStates_\dfa{C})$
    its equivalent minimal causal DFA with states $\dfaStates$,
    initial state $\dfaInitialState$,
    a set of accepting states $\dfaAcceptingStates_\dfa{C} \subseteq \dfaStates$,
    and transition function $\dfaTransitions$.
    Let $\dfaRejectingSinks \subseteq \dfaStates \setminus \dfaAcceptingStates_\dfa{C}$
    be the set of rejecting sink states of $\dfa{C}$.

    We define the product $\causalDiagram \times \prm{A}$ as a new PRM
    $(\prmStates, \prmInitState, \prmInputAlphabet, \prmOutputAlphabet, \prmTransitions, \prmOutput, \prmTerminals)$,
    where

    \begin{enumerate}
        \item $\prmStates = \prmStates^\prm{A} \times \dfaStates$,
              a state of $\causalDiagram \times \prm{A}$ is a pair of states $(\prmCommonState, \dfaCommonState)$ with
              $\prmCommonState \in \prmStates^\prm{A}$ and $\dfaCommonState \in \dfaStates$;
        \item $\prmInitState = (\prmInitState^\prm{A}, \dfaInitialState)$,
              the initial state in $\causalDiagram \times \prm{A}$ is the pair of initial states of $\prm{A}$ and $\dfa{C}$;
        \item $\prmOutputAlphabet = \prmOutputAlphabet^\prm{A} \cup \{m\}$,
              the output alphabet of $\causalDiagram \times \prm{A}$ is expanded with a possible reward output that is $1$ less than any output in the original set of rewards from $\prm{A}$;
        \item $\prmTransitions((\prmCommonState, \dfaCommonState), \propInput, (\prmCommonState', \dfaCommonState'))
                  = \prmTransitions^\prm{A}(\prmCommonState, \propInput, \prmCommonState')
                  \cdot \indicator{\dfaTransitions(\dfaCommonState, \propInput) = \dfaCommonState'}$,
              the probability of $\causalDiagram \times \prm{A}$ transitioning from $(\prmCommonState, \dfaCommonState)$ to
              $(\prmCommonState', \dfaCommonState')$ upon reading $\propInput$ is the same as
              the probability of $\prm{A}$ transitioning from $\prmCommonState$ to $\prmCommonState'$,
              given that $\dfa{C}$ transitions from $\dfaCommonState$ to $\dfaCommonState'$
              (otherwise, the probability is $0$);
        \item $\prmOutput((\prmCommonState, \dfaCommonState), \propInput, (\prmCommonState', \dfaCommonState')) =
                  \begin{cases}
                      \prmOutput^\prm{A}(\prmCommonState, \propInput, \prmCommonState')                                                                                  & \dfaCommonState' \not\in \dfaRejectingSinks \\
                      m = -1 - \max_{r \in \prmOutputAlphabet^\prm{A}}{|r|} - \max_{\prmCommonState \in \prmStates^\prm{A}}{\optimalStateValueFunction(\prmCommonState)} & \text{otherwise}
                  \end{cases}$,
              the output of the product PRM agrees with $\prm{A}$ except when $\dfa{C}$ transitions into a rejecting sink state; and
        \item $\prmTerminals = \{(\prmCommonState, \dfaCommonState) : \prmCommonState \in \prmTerminals^\prm{A}\}$,
              terminal states in $\causalDiagram \times \prm{A}$ correspond to terminal states in $\prm{A}$.
    \end{enumerate}

    \label{def:product}
\end{definition}

Performing value iteration acts as a form of look-ahead in the product $\causalDiagram \times \prm{A}$,
whose output function is defined so
that transitions which lead the causal DFA into a rejecting sink state do not contribute to overall state value.
The same is true for $B_2 = \causalDiagram \times (-\prm{A})$, which is defined in the same way, except the output function $-\prmOutput^\prm{A}(\prmCommonState, \propInput, \prmCommonState')$ provides look-ahead information about the worst-case future outcome.
Our method, given in Algorithm~\ref{alg:causal}, improves the convergence speed of QRM
by utilizing information about expected rewards that better reflects the temporal causal structure of the environment.
In Theorem~\ref{thm:convergence} we show that our method converges to the optimal policy in the limit.

\input{figures/product-fragment}

\begin{algorithm}
    \caption{Reinforcement Learning With Temporal-Causal Information}
    \label{alg:causal}
    \SetAlgoLined
    \SetKwInOut{Require}{Require}
    \SetKwFor{For}{for each}{do}{end for}
    \SetKwFor{While}{while}{do}{end while}

    \Require{MDP $\mdp{M}$, PRM $\prm{A}$, minimal causal DFA $\dfa{C}$ with rejecting sink states $\dfaRejectingSinks$}

    $\prm{B_1}, \prm{B_2} \leftarrow \text{computeProduct}(\prm{A}, \dfa{C})$, $\text{computeProduct}(-\prm{A}, \dfa{C})$ \label{line:product}

    $\optimalStateValueFunction_\prm{B_1}, \optimalStateValueFunction_\prm{B_2} \leftarrow \text{valueIteration}(\prm{B_1}, \mdpDiscount)$, $\text{valueIteration}(\prm{B_2}, \mdpDiscount)$ \label{line:value-iteration}

    $\prm{B} \leftarrow \prm{B_1}$ \label{line:b-b1}

    \label{line:t3-begin} \For{$\prmCommonState \in \prmStates^\prm{B}$}
    {
        \If{$\optimalStateValueFunction_\prm{B_1}(\prmCommonState) = \optimalStateValueFunction_\prm{B_2}(\prmCommonState) = 0$}
        {
            Add $\prmCommonState$ to the set of terminal states of $\prm{B}$ \label{line:t3-end}
        }
    }

    $Q \leftarrow \text{initializeQFunction}()$

    \While{\text{termination criteria not met}}
    {
        $Q \leftarrow \text{RunQRMEpisode}(Q, \prm{B})$
    }
    \Return $Q$
\end{algorithm}

\begin{mytheorem}[Convergence to Optimal Policy]
    Let $\mdp{M}$ be an MDP with a non-Markovian reward function captured by PRM $\prm{A}$.
    Let $\causalDiagram$ be a TL-CD that holds for $\mdp{M}$,
    and $\dfa{C}$ the corresponding minimal causal DFA with rejecting sink states $\dfaRejectingSinks$.
    Then Algorithm~\ref{alg:causal} converges to an optimal policy for $\mdp{M}$ with respect to $\prm{A}$.
    In particular, we can easily recover the optimal policy for $(\mdp{M}, \prm{A})$ from the optimal policy for $(\mdp{M}, \prm{B})$ found in the algorithm.
    \label{thm:convergence}
\end{mytheorem}

In the full proof of Theorem~\ref{thm:convergence},
we introduce $3$ transformations on PRMs that realize our method of combining a PRM with a TL-CD.
In brief, these transformations allow us to
(1) take the parallel composition of a PRM and a DFA,
(2) change the outputs of PRMs on transitions into unreachable states, and
(3) add states into the set of terminal states of the PRM (under certain conditions);
all the while preserving the optimal policy.
We then show how Algorithm~\ref{alg:causal} applies these transformations in order to arrive at the desired PRM $\prm{B}$.
Since these transformations preserve the optimal policy
(or allow for easy recovery of it),
we conclude that QRM using the transformed PRM $\prm{B}$ converges to the optimal policy.
In brief, the core contribution of the modifications we propose lies in the ability to exploit causal knowledge contained in the TL-CD.
Although we significantly change the structure of PRM by combining it with the TL-CD,
we prove that the optimal policy remains the same.
However, the changed structure allows for a more nuanced calculation of PRM state values that is not blind to the temporal-causal relations that hold for the MDP.
More precisely, we are able to obtain upper and lower bounds on state values,
and prove that under certain conditions one does not need to explore the MDP
(specifically, when both the upper and lower value bounds are $0$).
In doing so, we improve the balance between exploration and exploitation and increase the sample efficiency of our algorithm.
See the Appendix for further details about the function that computes the PRM and causal DFA intermediate product, and the full proof of Theorem~\ref{thm:convergence}.

%% file: figures/causal-dfa.tex
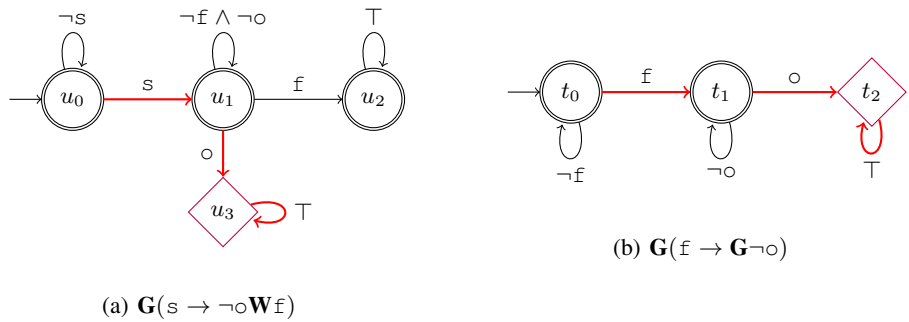
\begin{figure}[ht]
    \centering
    \subfigure[$\globallyOp(\texttt{s} \rightarrow \lnot \texttt{o} \weakUntilOp \texttt{f})$ ]{%
    \begin{minipage}{0.45\textwidth}
    \centering
    \begin{tikzpicture}[node distance=2cm]
    \node[state, initial, accepting, initial text=] (u0) {$u_0$};
    \node[state,accepting] (u1) [right of=u0] {$u_1$};
    \node[state, accepting] (u2) [right of=u1] {$u_2$};
    \node[state,diamond, draw=purple] (u3) at (2, -1.5) {$u_3$};
    \path (u0) edge[->, draw=red, thick] node[above]{$\texttt{s}$} (u1);
    \path (u1) edge[->] node[above]{$\texttt{f}$} (u2);
    \path (u1) edge[->, draw=red, thick] node[left]{$\texttt{o}$} (u3);
    \path (u0) edge[loop above] node[above]{$\lnot \texttt{s}$} (u0);
    \path (u1) edge[loop above] node[above]{$\lnot \texttt{f} \land \lnot \texttt{o}$} (u1);
    \path (u2) edge[loop above] node[above]{$\top$} (u2);
    \path (u3) edge[loop right, draw=red, thick] node[right]{$\top$} (u3);
\end{tikzpicture}
    \label{fig:causal-dfa-a}
    \end{minipage}
    }
    \hfill
    \subfigure[$\globallyOp(\texttt{f} \rightarrow \globallyOp \lnot \texttt{o})$]{%
    \begin{minipage}{0.45\textwidth}
    \centering
    \begin{tikzpicture}[node distance=2cm]
        \node[state, initial, accepting, initial text=] (t0) {$t_0$};
        \node[state] (t1) [right of=u0, accepting] {$t_1$};
        \node[state, diamond, draw=purple] (t2) [right of=t1] {$t_2$};
        \path (t0) edge[->, draw=red, thick] node[above]{$\texttt{f}$} (t1);
        \path (t1) edge[->, draw=red, thick] node[above]{$\texttt{o}$} (t2);
        \path (t0) edge[loop below] node[below]{$\lnot \texttt{f}$} (t0);
        \path (t1) edge[loop below] node[below]{$\lnot \texttt{o}$} (t1);
        \path (t2) edge[loop below, draw=red, thick] node[below]{$\top$} (t2);
    \end{tikzpicture}
    \label{fig:causal-dfa-b}
    \end{minipage}
    }
    \caption{Two factors of the causal DFA for the TL-CD in Figure~\ref{fig:cd-soda}.
    Rejecting sink states are diamond-shaped.
    Their parallel composition is the true causal DFA,
    and its states come from the Cartesian product of states in this Figure.
    For example, the initial state is $(u_0, t_0)$.
    }
    \label{fig:causal-dfa}
\end{figure}

%% file: figures/product-fragment.tex
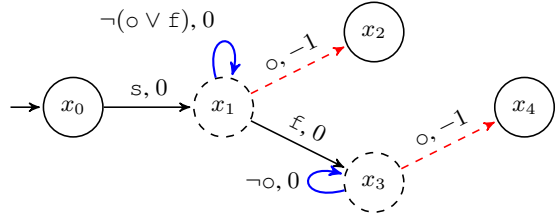
\begin{figure}[ht]
\centering
\begin{tikzpicture}[->,>=stealth',shorten >=1pt,auto,node distance=1.7cm,semithick]
\node[state, initial, initial text=] (x0) at (-2, 0) {$x_0$};
\node[state, draw=black, dashed] (x1) at (0, 0) {$x_1$};
\node[state] (x2) at (2, 1) {$x_2$};
\node[state, draw=black, dashed] (x3) at (2, -1) {$x_3$};
\node[state] (x4) at (4, 0) {$x_4$};

\path (x0) edge[] node[sloped]{$\texttt{s}$, $0$} (x1);
\path (x1) edge[draw=red, dashed] node[sloped]{$\texttt{o}$, $-1$} (x2);
\path (x1) edge[] node[sloped]{$\texttt{f}$, $0$} (x3);
\path (x3) edge[draw=red, dashed] node[sloped]{$\texttt{o}$, $-1$} (x4);

\path (x1) edge[loop above,draw=blue,thick] node[above left]{$\lnot(\texttt{o} \lor \texttt{f})$, $0$} (x1);
\path (x3) edge[loop left,draw=blue,thick] node[left]{$\lnot\texttt{o}$, $0$} (x3);


\end{tikzpicture}
\caption{A fragment of the product of the PRM from Figure~\ref{fig:prm-coffee-soda} and the TL-CD from Figure~\ref{fig:cd-soda}. Inheriting the $q$, $u$, and $t$ names of PRM and causal DFA states from previous figures,
$x_0 = (q_0, u_0, t_0)$,
$x_1 = (q_3, u_1, t_0)$,
$x_2 = (q_4, u_3, t_0)$,
$x_3 = (q_3, u_2, t_1)$, and
$x_4 = (q_4, u_2, t_2)$.
Due to the maximum in Equation~\ref{eqn:state-value},
dashed transitions do not contribute to state value.
Dashed states $x_1$ and $x_3$ have $0$ value in both $B_1$ (depicted) and $B_2$,
and will be added to the set of terminal states. }
\label{fig:product-fragment}
\end{figure}

%% file: sections/4_case_studies.tex
\section{Case Studies}
\label{sec:case-studies}

Our method shows promising results across two case studies.
The first case study (results in Figure~\ref{fig:results-comparison-a}) is based on the \emph{coffee vs. soda} example from Figure~\ref{fig:mdp-coffee-soda}.
The second case study (results in Figure~\ref{fig:results-comparison-b}) is described in Figure~\ref{fig:mdp-collection}.

\input{figures/results-comparison}

We compared our method against QRM without access to knowledge about causality.
In both case studies, our method takes significantly fewer steps to converge to the optimal policy.

\input{figures/mdp-collection_second_case}

For a more thorough comparison and analysis of the method's efficiency, we implemented it across two distinct case studies: a four-door task and a small office world domain. The third case study entails an agent navigating through a scenario where it must open four doors in any arbitrary order, as illustrated in Figure~\ref{fig:grid6-for-4door}. This task involves a significantly more complex PRM owing to the number of possible orders. To evaluate the method's performance and its efficacy in this case study, we use a grid world configuration of 6 $\times$ 6.

\input{figures/4door-cusaldfa-grid}
\input{figures/small_office_all}

The agent here must open door A, door B, door C, and door D in any order. However, door D is a trap, and the agent cannot see doors A, B, or even C after seeing door D. This knowledge, in fact, is encoded in Figure \ref{fig:tl-cd_4doors}. As this task requires a complex Probabilistic Reward Machine (PRM), we deemed it prudent to relegate its detailed explanation to the Appendix.

Furthermore, Figure \ref{fig:results-prm-4-doors} compares our method on the four-doors task to QRM without additional causal information. It can be seen that QRM with causal information results in much higher rewards with faster convergence.




Another case study in which we implemented this method is the small office world domain. For this specific exploration, we considered a small office world with a spatial layout of 17 $\times$ 9, similar to the setup in Paliwal et al.~\cite{10.1007/978-3-031-40837-3_8}.
Within the scope of this case study, the procedure to exit the grid entails a two-step process for the agent: first, it must obtain one of the two available keys, denoted as $k_1$ or $k_2$, and then navigate to exit $e_1$ or $e_2$, correspondingly aligned with the key acquired. Through one-way doors (indicated by blue arrows), keys, and walls, the agent interacts with the environment. A graphical illustration of this environment, capturing the elements and challenges the agent faces, is provided in Figure \ref{fig:map-small-office-1}, providing a better understanding of the structural and operational complexities of the small office world being explored.

As a result of $c$ being a one-way door, the agent will not be able to pick up key $k_2$ and exit at $e_2$, due to the information encoded in Figure~\ref{fig:tl-cd_small-office}.
In addition, if the agent passes through the door $b$, it will not be able to exit through the door $e_1$.
Furthermore, Figure~\ref{fig:prm-small-office} displays the PRM, omitting the causal information regarding the small office world.
In order to succeed in exiting the maze and receiving reward 1, the agent must complete both sequences a$-k_1-e_1$ (open door a, pick up key $k_1$, and leave at $e_1$) or b$-k_2-e_2$ (open door b, pick up key $k_2$, and exit at $e_2$). However, a probability of 0.9 suggests a likelihood of the agent exiting through $e_1$, while a probability of 0.1 indicates a risk of the agent getting stuck.

Figure~\ref{fig:results-prm-small-office} depicts the performance comparison of our method on the small office world scenario to QRM without additional causal information. In the figure, it can be seen that if the RL agent knows the causal DFA and learns never to open door $b$, the agent can obtain their optimal reward faster with higher accumulated rewards.

\input{figures/fig_case3and4}

\subsection{Performance Impact of Proposed Modifications}

Our method makes significant structural changes to the underlying PRM
in order to incorporate temporal-causal information contained in TL-CDs.
One of the primary effects of these changes is increasing the size of the state space.
Despite this increase in size,
experimental results demonstrate that our algorithm improves convergence properties on realistic examples
when the provided information is \emph{useful}
(allows for pruning redundant paths from the PRM where exploration is not necessary).
Strictly speaking, however, we only require that the causal diagrams are \emph{correct} (that they hold for the MDP),
not that they contain useful information.

\input{figures/results-spurious}

Being robust to mistakes and imprecisions in user-provided causal information (or even malicious inputs) is beneficial.
We hypothesize that our algorithm possesses this property,
in the sense that its performance does not diminish in the presence of useless or redundant knowledge.
In order to check this, we re-ran our experiments in this more difficult setting of perturbed user input.
We added useless and redundant causal knowledge,
by including an additional factor in the causal DFA product.
More precisely, we called \texttt{computeProduct} an additional time with a redundant causal DFA with no rejecting sink states.
This factor injects no new useful causal information,
but increases the state space size $5$ times.
The results we obtained were fully in line with the ones presented in our case studies,
i.e. the improved convergence rate enabled by our method was retained.
In conclusion, our algorithm can handle useless or redundant knowledge,
and its performance is not diminished by it.
We showcase these results in Figure~\ref{fig:results-spurious}.

We make two additional observations relating to state space size.

\begin{enumerate}
    \item PRMs themselves already provide an immense reduction in state space size, because the policy no longer has to consider the whole history. Instead, the PRM state acts as a form of finite memory in the MDP.
    \item Our method is based on removing redundant paths from the PRM (those where exploration is not necessary).
\end{enumerate}

We want to further drive home the point that while using PRMs and incorporating causal information as we propose does contribute to an increased state space,
the resulting performance benefits more than outweigh the costs
(and, as our additional analysis shows, some costs are fully avoided).

%% file: figures/results-comparison.tex
\begin{figure}[ht]
    \centering
    \subfigure[Coffee vs. soda task]{%
        \begin{minipage}{0.45\textwidth}
            \centering
            \begin{tikzpicture}[scale=1]
                \begin{axis}[
                        width=\linewidth, 
                        xlabel={Step},
                        ylabel={Reward per step},
                        grid=major,
                        legend style={
                                at={(0.05,0.95)},
                                anchor=north west,
                                font=\tiny,
                            },
                    ]
                    \addplot [mark=square*, blue] table [col sep=comma, x=Step, y=Value] {results/soda_qrm_results_without_causal.csv};
                    \addlegendentry{No causal}
                    \addplot [mark=*, red, densely dashed] table [col sep=comma, x=Step, y=Value] {results/soda_qrm_results_including_causal.csv};
                    \addlegendentry{Causal}
                \end{axis}
            \end{tikzpicture}\\
            \label{fig:results-comparison-a}
        \end{minipage}
    }
    \hfill 
    \subfigure[Two-doors task]{%
        \begin{minipage}{0.45\textwidth}
            \centering
            \begin{tikzpicture}[scale=1]
                \begin{axis}[
                        width=\linewidth, 
                        xlabel={Step},
                        grid=major,
                    ]
                    \addplot [mark=square*, blue] table [col sep=comma, x=Step, y=Value] {results/collection_qrm_results_without_causal.csv};
                    \addplot [mark=*, red, densely dashed] table [col sep=comma, x=Step, y=Value] {results/collection_qrm_results_including_causal.csv};
                \end{axis}
            \end{tikzpicture}\\
            \label{fig:results-comparison-b}
        \end{minipage}
    }
    \caption{\hadi{Reward} per step averaged over $20$ runs. ``No causal'' refers to using QRM with the original PRM that does not account for additional causal information in the environment. ``Causal'' are the results for our method. Both graphs showcase QRM convergence to the optimal policy.}
    \label{fig:results-comparison}
\end{figure}
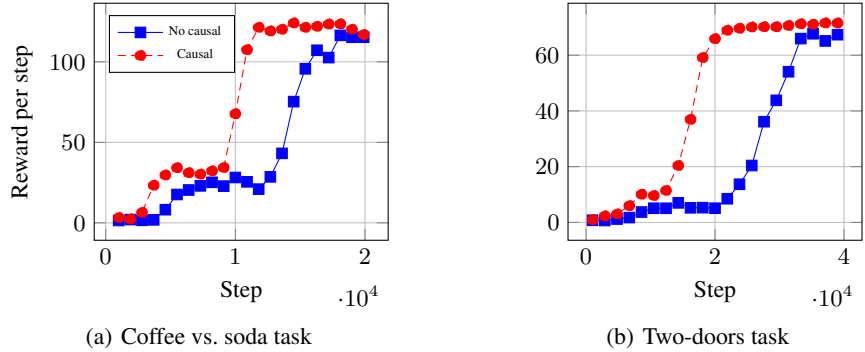

%% file: figures/mdp-collection_second_case.tex
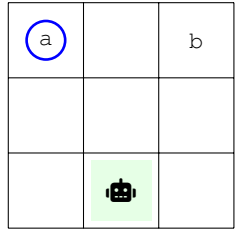
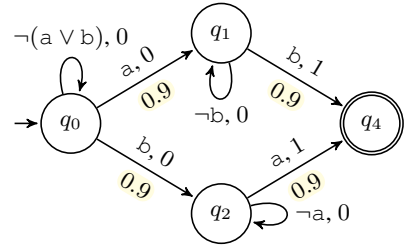
\begin{figure}
  \centering
  \subfigure[\texttt{3x3} Gridworld environment where the agent must open both door A (\texttt{a}) and door B (\texttt{b}) in any order. However, the cell with door A traps the agent.
    The agent can fail at opening the doors with probability $0.1$.]{%
    \begin{minipage}{0.45\textwidth}
      \centering
      \begin{tikzpicture}[scale=1]
        \draw (0,0) grid (3,3);
        \fill[green!10] (1.1,0.1) rectangle (1.9,0.9);
        \node at (1.5,0.5) {\startLabel{}};
        \node at (0.5,2.5) {\tikz[baseline=(char.base)]{\node[line width=1pt,blue,draw,circle,inner sep=2.8pt, text=black] (char) {\texttt{a}};}};
        \node at (2.5,2.5) {\texttt{b}};
      \end{tikzpicture}
      \label{fig:mdp-collection_a}
    \end{minipage}
  }
  \hfill
  \subfigure[The PRM without causal info about the two-door task. Missing transitions are all self-loops with probability $0.1$.]{%
    \begin{minipage}{0.45\textwidth}
      \centering
      \begin{tikzpicture}[->,>=stealth',shorten >=1pt,auto,node distance=3cm,semithick, scale=0.8]
        \node[state, initial, initial text=] (q0) at (-0.5, 0) {$q_0$};
        \node[state] (q1) at (2, 1.5) {$q_1$};
        \node[state] (q2) at (2, -1.5) {$q_2$};
        \node[state, accepting] (q3) at (4.5, 0) {$q_4$};
        \path (q0) edge[] node[sloped]{$\texttt{a}$, $0$} node[sloped, swap]{\highlightyellow{$0.9$}} (q1);
        \path (q0) edge[] node[sloped]{$\texttt{b}$, $0$} node[sloped, swap]{\highlightyellow{$0.9$}} (q2);
        \path (q1) edge[] node[sloped]{$\texttt{b}$, $1$} node[sloped, swap]{\highlightyellow{$0.9$}} (q3);
        \path (q2) edge[] node[sloped]{$\texttt{a}$, $1$} node[sloped, swap]{\highlightyellow{$0.9$}} (q3);
        \path (q0) edge[loop above] node[above]{$\lnot(\texttt{a} \lor \texttt{b})$, $0$} (q0);
        \path (q1) edge[loop below] node[below]{$\lnot\texttt{b}$, $0$} (q1);
        \path (q2) edge[loop right] node[right]{$\lnot\texttt{a}$, $0$} (q2);
      \end{tikzpicture}
      \label{fig:pcd-collection-b}
    \end{minipage}
  }
  \caption{The MDP and PRM for the second case study. The TL-CD that adds causal information regarding the sink door A can be found on Figure~\ref{fig:cd-collection}.
    It states that after seeing door A, the agent can not later see door B.
  }
  \label{fig:mdp-collection_second_case}
\end{figure}

%% file: figures/4door-cusaldfa-grid.tex
\begin{figure}[ht]
    \centering
    \subfigure[\texttt{6x6} Gridworld environment where the agent must open door A (\texttt{a}), door B (\texttt{b}), door C (\texttt{c}), and door D (\texttt{d}) in any order. However, the cell with door D traps the agent. The agent may fail to open door B with a probability of 0.1, and it will go to door D instead.]{%
    \begin{minipage}{0.45\textwidth}
    \centering
    \begin{tikzpicture}[scale=0.6]
            \draw (0,0) grid (6,6);
            \fill[green!10] (3.1,3.1) rectangle (3.9,3.9);
            \node at (3.5,3.5) {\startLabel{}};
            \node at (5.5,2.5) {\tikz[baseline=(char.base)]{\node[line width=1.0pt,blue,draw,circle,inner sep=2.6pt, text=black] (char) {\texttt{d}};}}; 
            \node at (5.5,5.5) {\texttt{c}};      
            \node at (5.5,0.5) {\texttt{b}};
            \node at (0.5,0.5) {\texttt{a}};
        \end{tikzpicture}
    \label{fig:grid6-for-4door}
    \end{minipage}
    }
    \hfill
    \subfigure[TL-CD that holds for the case study of four-doors task.]{%
    \begin{minipage}{0.45\textwidth}
    \centering
    \raisebox{0.5cm}{%
        \begin{tikzpicture}[scale = 0.9]
          \node[draw, ultra thick, circle] (p) at (-1, 0) {$\texttt{d}$};
          \node[draw, ultra thick, circle] (q) at (2, 0) {$\globallyOp \lnot (\texttt{a} \lor \texttt{b} \lor \texttt{c})$};
          \draw[-latex] (p) -- (q) node[midway, above] {};
        \end{tikzpicture}
        }
    \label{fig:tl-cd_4doors}
    \end{minipage}
    }
    \caption{The MDP and Causal DFA for the third case study.
    }
    \label{fig:4door-cusaldfa-grid}
\end{figure}
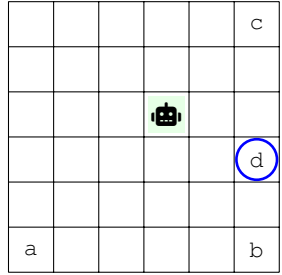
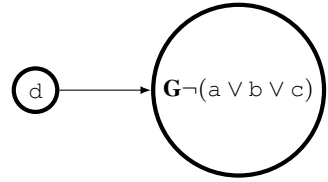

%% file: figures/small_office_all.tex
\begin{figure}[ht]
    \centering
    \subfigure[Map of the small office world. Shaded cells are impassable walls.]{%
        \begin{minipage}{0.9\textwidth}
            \centering
            \begin{tikzpicture}[scale=0.5]
                \draw[step=1cm,black!50] (0,0) grid (17,9);
                \fill[green!10] (8,4) rectangle (9,5);
                \node at (8.5,4.5) {\startLabel{}};
                \node at (0.5,8.5) {$e_1$};
                \node at (12.5,8.5) {$e_2$};
                \node at (2.5,5.5) {$k_1$};
                \node at (15.5,7.5) {$k_2$};
                \node at (9.5,4.5) {$b$};
                \node at (7.5,4.5) {$a$};
                \node at (14.5,4.5) {$c$};
                \fill[red!10] (2+0.1,3+0.1) rectangle (5-0.1,4-0.1);
                \fill[red!10] (4+0.1,4+0.1) rectangle (5-0.1,7-0.1);
                \fill[red!10] (6+0.1,9-0.1) rectangle (7-0.1,5+0.1);
                \fill[red!10] (6+0.1,6-0.1) rectangle (11-0.1,5+0.1);
                \fill[red!10] (10+0.1,9-0.1) rectangle (11-0.1,5+0.1);
                \fill[red!10] (6+0.1,0+0.1) rectangle (7-0.1,4-0.1);
                \fill[red!10] (6+0.1,3+0.1) rectangle (11-0.1,4-0.1);
                \fill[red!10] (10+0.1,0+0.1) rectangle (11-0.1,4-0.1);
                \fill[red!10] (13+0.1,9-0.1) rectangle (14-0.1,5+0.1);
                \fill[red!10] (13+0.1,4-0.1) rectangle (14-0.1,2+0.1);
                \fill[red!10] (13+0.1,3-0.1) rectangle (15-0.1,2+0.1);
                \fill[red!10] (14+0.1,3-0.1) rectangle (15-0.1,0+0.1);
                \draw[->,thick,blue] (0.5,9) -- (0.5,9.5);
                \draw[->,thick,blue] (12.5,9) -- (12.5,9.5);
                \draw[->,thick,blue] (8.1,4.5) -- (6.5,4.5);
                \draw[->,thick,blue] (8.9,4.5) -- (10.5,4.5);
                \draw[orange, thick] (8,5) -- (8,4);
                \draw[orange, thick] (9,5) -- (9,4);
                \draw[orange, thick] (13,5) -- (13,4);
                \draw[blue, thick] (13,5) -- (15,5);
                \draw[blue, thick] (13,4) -- (16,4);
                \draw[blue, thick] (15,5) -- (15,8);
                \draw[blue, thick] (16,4) -- (16,8);
                \draw[->,thick,blue] (13.1,4.5) -- (13.7,4.5);
                \draw[->,thick,blue] (15.5,4.5) -- (15.5,4.9);
                \draw[->,thick,blue] (15.5,4.5) -- (15.5,4.9);
                \draw[->,thick,blue] (15.5,5.3) -- (15.5,6);
                \draw[->,thick,blue] (15.5,6.3) -- (15.5,7);
            \end{tikzpicture}
            \label{fig:map-small-office-1}
        \end{minipage}
    }

    \subfigure[The PRM without causal info about the small office world.]{%
        \begin{minipage}{0.45\textwidth}
            \centering
            \begin{tikzpicture}[->,>=stealth',shorten >=1pt,auto,node distance=3cm,semithick, scale=0.6]
                \node[state, initial, initial text=] (q0) at (0, 0) {$q_0$};
                \node[state] (q1) at (3, 2) {$q_1$};
                \node[state] (q2) at (5.5, 2) {$q_2$};
                \node[state] (q3) at (3, -2) {$q_3$};
                \node[state] (q4) at (5.5, -2) {$q_4$};
                \node[state, accepting] (q5) at (8, 0) {$q_5$};
                \path (q0) edge[loop below] node[below]{$\lnot(\texttt{a} \lor \texttt{b} )$, $0$}node[below, yshift=-3ex]{\highlightyellow{$1$}}(q0);
                \path (q0) edge[] node[sloped]{$\texttt{a} \land \lnot\texttt{b}$, $0$} node[sloped, swap]{\highlightyellow{$1$}} (q1);
                \path (q0) edge[] node[sloped]{$\texttt{b}$, $0$} node[sloped, swap]{\highlightyellow{$1$}} (q3);
                \path (q1) edge[loop above] node[above, yshift=2.5ex]{$\lnot \texttt{k}_{1}$, $0$}node[below, yshift=3ex]{\highlightyellow{$1$}}(q1);
                \path (q1) edge[] node[sloped]{$\texttt{k}_{1}$, $0$} node[sloped, swap]{\highlightyellow{$1$}} (q2);
                \path (q2) edge[loop above] node[above, yshift=2.5ex]{$\lnot \texttt{e}_{1}$, $0$}node[below, yshift=3ex]{\highlightyellow{$1$}}(q2);
                \path (q2) edge[loop below] node[below, yshift=-0.3ex]{$ \texttt{e}_{1}$, $0$}node[below, yshift=-3ex]{\highlightyellow{$0.1$}}(q2);
                \path (q2) edge[] node[sloped]{$\texttt{e}_{1}$, $1$} node[sloped, swap]{\highlightyellow{$0.9$}} (q5);
                \path (q3) edge[loop below] node[above, yshift=-3ex]{$\lnot \texttt{k}_{2}$, $0$}node[below, yshift=-2.5ex]{\highlightyellow{$1$}}(q3);
                \path (q3) edge[] node[sloped]{$\texttt{k}_{2}$, $0$} node[sloped, swap]{\highlightyellow{$1$}} (q4);
                \path (q4) edge[loop below] node[above, yshift=-3ex]{$\lnot \texttt{e}_{2}$, $0$}node[below, yshift=-2.5ex]{\highlightyellow{$1$}}(q4);
                \path (q4) edge[] node[sloped]{$\texttt{e}_{2}$, $1$} node[sloped, swap]{\highlightyellow{$1$}} (q5);
                \path (q5) edge[loop above] node[above]{$\top$}node[below]{\highlightyellow{$ $}}(q5);
            \end{tikzpicture}
            \label{fig:prm-small-office}
        \end{minipage}
    }
    \hfill
    \subfigure[TL-CD that holds for the case study of the small office.]{%
        \begin{minipage}{0.45\textwidth}
            \centering
            \begin{tikzpicture}[scale = 0.85]
                \node[draw, ultra thick, circle] (p1) at (-4, 0) {$\texttt{b}$};
                \node[draw, ultra thick, circle] (q1) at (-4, -2) {$\globallyOp \lnot \texttt{e}_{1}$};
                \draw[-latex] (p1) -- (q1) node[midway, left] {};
                \node[draw, ultra thick, circle] (p2) at (-2, 0.5) {$\texttt{c}$};
                \node[draw, ultra thick, ellipse, minimum width=1.2cm, minimum height=0.8cm](q2) at (-2, -1.5) {$\nextOp\nextOp\nextOp\nextOp k_2$};
                \draw[-latex] (p2) -- (q2) node[midway, left] {};
                \node[draw, ultra thick, circle] (p3) at (0, 0) {$\texttt{k}_{2}$};
                \node[draw, ultra thick, circle] (q3) at (0, -2) {$\globallyOp \lnot \texttt{e}_{2}$};
                \draw[-latex] (p3) -- (q3) node[midway, left] {};
            \end{tikzpicture}
            \label{fig:tl-cd_small-office}
        \end{minipage}
    }
    \caption{\hadi{The PRM and Causal DFA for the fourth case study, alongside the office map for reference.}}
    \label{fig:office-case-study-combined}
\end{figure}
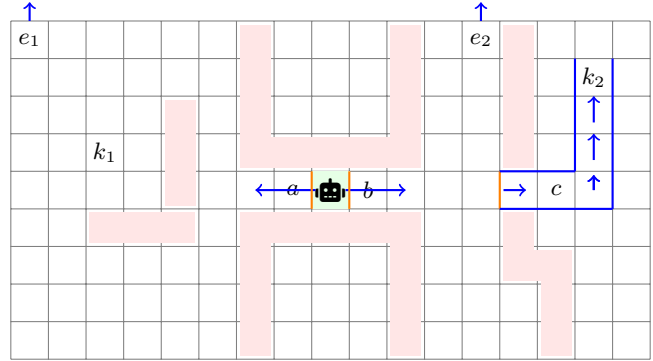
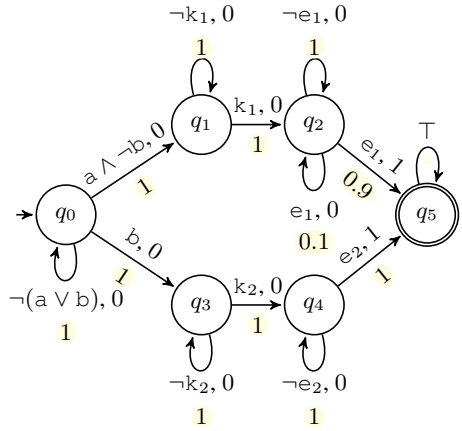
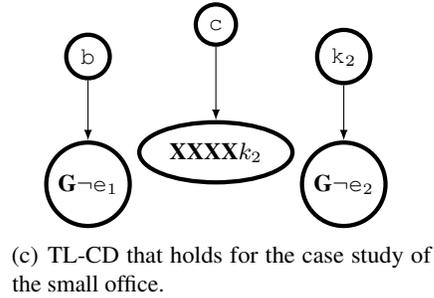

%% file: figures/fig_case3and4.tex
\begin{figure}[ht]
    \centering
    \subfigure[Four-doors task results.]{%
        \begin{minipage}{0.48\textwidth}
            \centering
            \begin{tikzpicture}[scale=1]
                \begin{axis}[
                        width=\linewidth, 
                        xlabel={Step},
                        ylabel={Reward per step},
                        grid=major,
                        legend style={
                                font=\tiny,
                                at={(0.6,0.8)},
                                anchor=north west,
                            },
                    ]
                    \addplot [mark=square*, blue] table [col sep=comma, x=Step, y=Value] {results/collection4_qrm_results_without_causal.csv};
                    \addlegendentry{No causal}
                    \addplot [mark=*, red,densely dashed] table [col sep=comma, x=Step, y=Value] {results/collection4_qrm_results_including_causal.csv};
                    \addlegendentry{Causal}
                \end{axis}
            \end{tikzpicture}
            \label{fig:results-prm-4-doors}
        \end{minipage}
    }
    \hfill 
    \subfigure[Small office world results.]{%
        \begin{minipage}{0.48\textwidth}
            \centering
            \begin{tikzpicture}[scale=1]
                \begin{axis}[
                        width=\linewidth, 
                        xlabel={Step},
                        ylabel style={overlay, anchor=north,yshift=10pt},
                        grid=major,
                    ]
                    \addplot [mark=square*, blue] table [col sep=comma, x=Step, y=Value] {results/collection3_qrm_results_without_causal.csv};
                    \addplot [mark=*, red, densely dashed] table [col sep=comma, x=Step, y=Value] {results/collection3_qrm_results_including_causal.csv};
                \end{axis}
            \end{tikzpicture}
            \label{fig:results-prm-small-office}
        \end{minipage}
    }
    \caption{\hadi{Comparison of task results, using the same reward per step metric averaged over 20 runs.}}
    \label{fig_case3&4}
\end{figure}
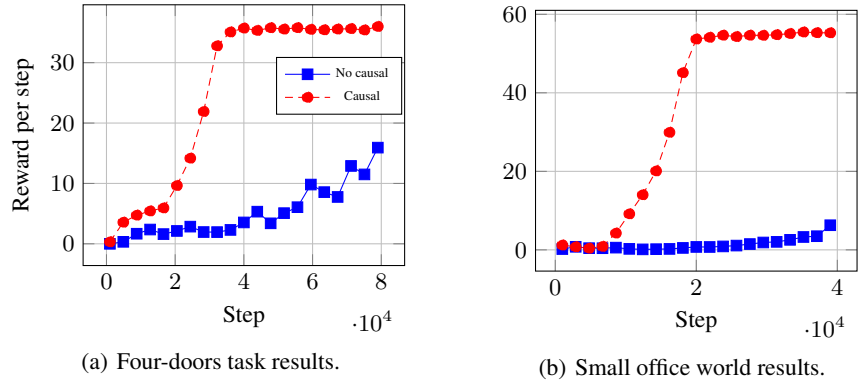

%% file: figures/results-spurious.tex
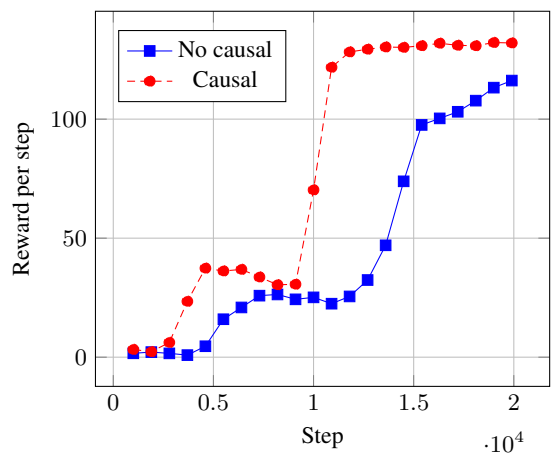
\begin{figure}[t]
    \centering
    \vspace{-2ex}
    \begin{tikzpicture}
        \begin{axis}[
                width=3in,
                xlabel={Step},
                ylabel={Reward per step},
                grid=major,
                legend style={
                        at={(0.05,0.95)}, 
                        anchor=north west, 
                    },
            ]
            \addplot [mark=square*, blue] table [col sep=comma, x=Step, y=Value] {results/soda_spurious_qrm_results_without_causal.csv};
            \addlegendentry{No causal}
            \addplot [mark=*, red, densely dashed] table [col sep=comma, x=Step, y=Value] {results/soda_spurious_qrm_results_including_causal.csv};
            \addlegendentry{Causal}
        \end{axis}
    \end{tikzpicture}
    \caption{
        The Coffee vs. soda task.
        ``Causal'' are the results for our method, but with additional, redundant causal knowledge added.
        The resulting PRM has $5$ times more states than the one from Figure~\ref{fig:results-comparison-a},
        but our method achieves the same performance.
        Like before, ``No causal'' refers to using QRM with the original PRM
        (which has neither causal information nor redundancy added).
    }
    \label{fig:results-spurious}
\end{figure}

%% file: sections/5_conclusion_and_further_work.tex
\section{Conclusion and Further Work}

The method proposed in this paper addresses the difficult problem of accounting for knowledge about temporal causality in the RL environment.
We have shown that an expressive and concise description of temporal and causal relations
in the form of a Temporal Logic-based Causal Diagram
can be integrated into the reward function formalism.
Furthermore, we have shown how the added information about temporal and causal relations can be leveraged to expedite learning
without changing the optimal policy.

While our method performs well in case studies,
we are convinced that this work can be continued to integrate knowledge about causality even more tightly into the reward function.
In particular, look-ahead information contained in state-values of the product PRM may be further utilized by methods like reward shaping.
We are also interested in further exploring the interplay between probabilistic outcomes and causal information.

%% file: appendix_sections/1_product.tex
\section{Computing the PRM and causal DFA product}
\label{sec:product}

Algorithm~\ref{alg:product} computes the (intermediate) product $\prm{B_1}$ ($\prm{B_2}$) of a PRM $\prm{A}$ and causal DFA $\dfa{C}$.
This function is called in Line~\ref{line:product} of Algorithm~\ref{alg:causal} in Section~\ref{sec:method}.
Note that this function does not perform value iteration,
this is done in Line~\ref{line:value-iteration} of Algorithm~\ref{alg:causal}.
The set of terminal states in the intermediate product $\prm{B_1}$ is a subset of the set of terminal states in the final product $\prm{B}$.

The $\texttt{transitions}$ dictionary used in Algorithm~\ref{alg:product}
represents the transition and output functions of $\prm{B_1}$.
It maps triplets
$((\prmCommonState, \dfaCommonState), \propInput, (\prmCommonState', \dfaCommonState'))$
to to pairs $(p, r)$,
where $p = \prmTransitions^{\prm{B_1}}((\prmCommonState, \dfaCommonState), \propInput, (\prmCommonState', \dfaCommonState'))$, and
$r = \prmOutput^{\prm{B_1}}((\prmCommonState, \dfaCommonState), \propInput, (\prmCommonState', \dfaCommonState'))$.

\begin{algorithm}
  \caption{computeProduct($\prm{A}$, $\dfa{C}$)}
  \SetAlgoLined

  \SetKwInOut{Input}{Input}
  \SetKwInOut{Output}{Output}
  \label{alg:product}
  \setcounter{AlgoLine}{0}
  \Input{PRM $\prm{A}$, minimal causal DFA $\dfa{C}$ with rejecting sink states $\dfaRejectingSinks$}

  $\text{appears} \gets \prm{A}.\text{appears} \cup \dfa{C}.\text{appears}$ \tcp*[r]{M.appears is the set of relevant atomic propositions in M.}

  $\text{pairToSelfStateMap} \gets \{\}$ \tcp*[r]{Dictionary mapping pairs of states in $\prm{A}$ and $\dfa{C}$ to a single state in $\prm{B}'$}

  $\text{selfToPairStateMap} \gets \{\}$

  $\text{nonTerminalStates} \gets \emptyset$ \tcp*[r]{Contains non-terminal states of the intermediate product PRM.}

  $\text{terminalStates} \gets \emptyset$ \tcp*[r]{Contains terminal states of the intermediate product PRM.}

  $\text{transitions} \gets \{\}$ \tcp*[r]{Dictionary representation of $\prmTransitions$ and $\prmOutput$ functions of $\prm{B}'$.}

  $\text{stateCounter} \gets 0$

  \For{$\prmCommonState$ in $\prm{A}.\text{states}$}{
    \For{$\dfaCommonState$ in $\dfa{C}.\text{states}$}{
      $\text{pairToSelfStateMap}[(\prmCommonState, \dfaCommonState)] \gets \text{stateCounter}$

      $\text{selfToPairStateMap}[\text{stateCounter}] \gets (\prmCommonState, \dfaCommonState)$

      \If{$\prmCommonState$ in $\prm{A}.\text{terminalStates}$}{
        $\text{terminalStates} \gets \{\text{stateCounter}\} \cup \text{terminalStates}$
      }
      \Else{
        $\text{nonTerminalStates} \gets \{\text{stateCounter}\} \cup \text{nonTerminalStates}$
      }
      $\text{stateCounter} \gets \text{stateCounter} + 1$
    }
  }

  \For{$\prmCommonState$ in $\prm{A}.\text{nonTerminalStates}$}{
    \For{$\dfaCommonState$ in $\dfa{C}.\text{states}$}{
      $\text{state} \gets \text{pairToSelfStateMap}[(\prmCommonState, \dfaCommonState)]$

      $\text{transitions}[\text{state}] \gets \{\}$

      \For{$\text{inputSymbol}$ in \Call{generateInputs}{\text{appears}}}{
        $\text{inputSymbolPrm} \gets \prm{A}.\text{appears} \cap \text{inputSymbol}$

        $\text{inputSymbolDfa} \gets \dfa{C}.\text{appears} \cap \text{inputSymbol}$

        $\text{transitions}[\text{state}][\text{inputSymbol}] \gets \{\}$

        $\text{nextDfaState} \gets \dfa{C}.\text{transitions}[\dfaCommonState][\text{inputSymbolDfa}]$

        \For{$\text{nextPrmState}$ in $\prm{A}.\text{transitions}[\prmCommonState][\text{inputSymbolPrm}]$}{
          $\text{nextState} \gets \text{pairToSelfStateMap}[(\text{nextPrmState}, \text{nextDfaState})]$

          $\text{probability}, \text{reward} \gets \prm{A}.\text{transitions}[\prmCommonState][\text{inputSymbolPrm}][\text{nextPrmState}]$

          \If{$\text{nextDfaState}$ in $\dfaRejectingSinks$}{
            $\text{reward} \gets m$
          }
          $\text{transitions}[\text{state}][\text{inputSymbol}][\text{nextState}] \gets (\text{probability}, \text{reward})$
        }
      }
    }
  }

  \Return PRM($\text{transitions}$, $\text{appears}$, $\text{terminalStates}$)
\end{algorithm}

%% file: appendix_sections/2_proofs.tex
\section{Formal Statements and Proofs}
\label{sec:theory}

Throughout this Section, let $\mdp{M} = (\mdpStates, \mdpActions, \mdpReward, \mdpDynamics, \labelingFunction, \mdpDiscount)$ be a labeled MDP,
and let $\prm{A} = (\prmStates, \prmInitState, \prmInputAlphabet, \prmOutputAlphabet, \prmTransitions, \prmOutput, \prmTerminals)$ be a PRM that encodes $\mdpReward$ in $\mdp{M}$. A state $\mdpCommonState \in \mdpStates$ is reachable if there exists an attainable trajectory $\mdpTrajectory{n}$ in $\mdp{M}$ such that $\mdpCommonState = \mdpCommonState_i$ for some $i = 0, \ldots, n$.


\begin{mylemma}[Transformation 1]
    \label{lemma:t1}
    Let $\prm{A} = (\prmStates, \prmInitState, \prmInputAlphabet, \prmOutputAlphabet, \prmTransitions, \prmOutput, \prmTerminals)$ be a PRM,
    $\dfa{C}$ a DFA with transition function $\dfaTransitions$,
    and $\prm{A} \times \dfa{C}$ their parallel composition with output function
    $\prmOutput^{\prm{A} \times \dfa{C}}((\prmCommonState, \dfaCommonState), \propInput, (\prmCommonState', \dfaCommonState')) = \prmOutput(\prmCommonState, \propInput, \prmCommonState')$.
    Let $\pi^\star(\mdpCommonState, (\prmCommonState, \dfaCommonState))$ be an optimal policy in the product MDP $\mdp{M} \times (\prm{A} \times \dfa{C})$.
    Then $\pi(\mdpCommonState, \prmCommonState) = \pi^\star(\mdpCommonState, (\prmCommonState, F(\prmCommonState)))$ is an optimal policy
    in the product MDP $\mdp{M} \times \prm{A}$,
    where $F_{\mdpCommonState} : \prmStates \to \dfaStates$ is a tiebreaking reachability function that maps $\prmCommonState$ to a fixed but arbitrary $\dfaCommonState$ such that $(\mdpCommonState, \prmCommonState, \dfaCommonState)$ is reachable in $\mdp{M} \times (\prm{A} \times \dfa{C})$.
\end{mylemma}

\begin{proof}
    Since Q-learning with an $\epsilon$-greedy exploration strategy visits every reachable state infinitely often in the limit,
    we can learn a mapping $F_s$ along with the optimal policy for
    $\mdp{M} \times (\prm{A} \times \dfa{C})$,
    by setting $F_{\mdpCommonState}(\prmCommonState) = \dfaCommonState$ once we reach $(\mdpCommonState, \prmCommonState, \dfaCommonState)$ for the first time, and we have not yet observed $(\mdpCommonState, \prmCommonState, \dfaCommonState')$ for any $\dfaCommonState' \in \dfaStates$.
    Then, from the fact that

    \begin{enumerate}
        \item Q-learning learns the optimal Q-function in all reachable states, and
        \item the optimal Q-function $\optimalQFunction(\mdpCommonState, \prmCommonState, \dfaCommonState, \mdpCommonAction)$ for $\mdp{M} \times (\prm{A} \times \dfa{C})$ is the same as $\optimalQFunction(\mdpCommonState, \prmCommonState, \mdpCommonAction)$ for $\mdp{M} \times \prm{A}$ once we project out the DFA component $\dfaCommonState \in \dfaStates$,
    \end{enumerate}

    it will follow that we can use $F_{\mdpCommonState}$ to construct the optimal policy in $\mdp{M} \times \prm{A}$.

    Let $\optimalQFunction(\mdpCommonState, \prmCommonState, \mdpCommonAction)$ be the optimal Q-function for $\mdp{M} \times \prm{A}$,
    and $\optimalQFunction(\mdpCommonState, \prmCommonState, \dfaCommonState, \mdpCommonAction)$ the optimal Q-function for $\mdp{M} \times (\prm{A} \times \dfa{C})$.
    We will proceed by showing that
    $\optimalQFunction(\mdpCommonState, \prmCommonState, \dfaCommonState_1, \mdpCommonAction) =
        \optimalQFunction(\mdpCommonState, \prmCommonState, \dfaCommonState_2, \mdpCommonAction)$ for all $\dfaCommonState_1, \dfaCommonState_2 \in \dfaStates$.

    Let $\tilde{\qFunction}(\mdpCommonState, \prmCommonState, \dfaCommonState, \mdpCommonAction) = \optimalQFunction(\mdpCommonState, \prmCommonState, \mdpCommonAction)$. We will show that $\tilde{\qFunction}$ is a solution to the state value Bellman optimality equation for $\mdp{M} \times (\prm{A} \times \dfa{C})$, given in Equation~\ref{eqn:bellman-extended}. In the following formulas, let $\propInput = \labelingFunction(\mdpCommonState, \mdpCommonAction, \mdpCommonState')$.

    \begin{equation}
        \begin{aligned}
            \optimalQFunction & (\mdpCommonState, \prmCommonState, \dfaCommonState, \mdpCommonAction) =                                                                                                                                                                                         \\
                              & = \sum_{\substack{\mdpCommonState' \in \mdpStates                                                                                                                                                                                                               \\ \prmCommonState' \in \prmStates \\ \dfaCommonState' \in \dfaStates}}
            \mdpDynamics(\mdpCommonState', \prmCommonState', \dfaCommonState' \mid \mdpCommonState, \prmCommonState, \dfaCommonState, \mdpCommonAction)
            \left(\prmOutput^{\prm{A} \times \dfa{C}}((\prmCommonState, \dfaCommonState), \propInput, (\prmCommonState', \dfaCommonState')) + \gamma \max_{\mdpCommonAction' \in \mdpActions} \optimalQFunction(\mdpCommonState', \prmCommonState', \dfaCommonState', \mdpCommonAction')\right) \\
                              & = \sum_{\substack{\mdpCommonState' \in \mdpStates                                                                                                                                                                                                               \\ \prmCommonState' \in \prmStates }}
            \mdpDynamics(\mdpCommonState', \prmCommonState' \mid \mdpCommonState, \prmCommonState, \mdpCommonAction)
            \left(\prmOutput^{\prm{A}}(\prmCommonState, \propInput, \prmCommonState') + \gamma \max_{\mdpCommonAction' \in \mdpActions} \optimalQFunction(\mdpCommonState', \prmCommonState', \dfaTransitions(\dfaCommonState, \propInput), \mdpCommonAction')\right)
        \end{aligned}
        \label{eqn:bellman-extended}
    \end{equation}

    The second equality holds because

    \begin{enumerate}
        \item $\mdpDynamics(\mdpCommonState', \prmCommonState', \dfaCommonState' \mid \mdpCommonState, \prmCommonState, \dfaCommonState, \mdpCommonAction) = 0$ for $\dfaTransitions(\dfaCommonState, \propInput) \neq \dfaCommonState'$,
        \item $\mdpDynamics(\mdpCommonState', \prmCommonState', \dfaCommonState' \mid \mdpCommonState, \prmCommonState, \dfaCommonState, \mdpCommonAction) = \mdpDynamics(\mdpCommonState', \prmCommonState' \mid \mdpCommonState, \prmCommonState, \mdpCommonAction)$ for $\dfaTransitions(\dfaCommonState, \propInput) = \dfaCommonState'$, and
        \item $\prmOutput^{\prm{A} \times \dfa{C}}((\prmCommonState, \dfaCommonState), \propInput, (\prmCommonState', \dfaCommonState')) = \prmOutput^{\prm{A}}(\prmCommonState, \propInput, \prmCommonState')$.
    \end{enumerate}

    by definition of $\prm{A} \times \dfa{C}$.
    Equation~\ref{eqn:bellman-extended} shows that the Bellman optimality equation for $\optimalQFunction(\mdpCommonState, \prmCommonState, \dfaCommonState, \mdpCommonAction)$  reduces to the Bellman optimality equation for $\optimalQFunction(\mdpCommonState, \prmCommonState, \mdpCommonAction)$.
    More precisely, the parameters for the system of nonlinear equations given in Equation~\ref{eqn:bellman-extended} are the same as those in the system for $\optimalQFunction(\mdpCommonState, \prmCommonState, \mdpCommonAction)$,
    except that each individual equation is repeated $|\dfaStates|$ times (once for every DFA state).
    Therefore, we have $\tilde{\qFunction}(\mdpCommonState, \prmCommonState, \dfaCommonState, \mdpCommonAction) = \optimalQFunction(\mdpCommonState, \prmCommonState, \mdpCommonAction)$ as a solution.

    Now all that is left is the fact that Q-learning in $\mdp{M} \times (\prm{A} \times \dfa{C})$ will converge to the optimal Q-function that is independent of the $\dfaCommonState \in \dfaStates$ component,
    and that Q-learning will converge to the same Q-function in $\mdp{M} \times \prm{A}$.
    Values in unreachable states will remain unaffected by learning updates,
    and will not affect the return from the optimal policy.


\end{proof}

\begin{definition}[Unreachable PRM state]
    Let $\prm{A} = (\prmStates, \prmInitState, \prmInputAlphabet, \prmOutputAlphabet, \prmTransitions, \prmOutput, \prmTerminals)$ be a PRM.
    A state $\prmCommonState \in \prmStates$ is $\mdp{M}$-unreachable if for every input sequence $\lambda$
    s.t. $\prm{A} \xrightarrow{\lambda} \prmCommonState$
    ($\prm{A}$ transitions into $\prmCommonState$ upon reading $\lambda$) we have that every trajectory $\mdpTrajectory{n}$ such that $\labelingFunction(\mdpTrajectory{n}) = \lambda$ is unattainable in $\mdp{M}$ (has probability $0$ according to the transition function $\mdpDynamics$ of $\mdp{M}$).
    \label{def:unreachable-prm-state}
\end{definition}

\begin{mylemma}[Transformation 2]
    Let $\alpha \in \reals$ be an arbitrary real.
    Let $\prm{A} = (\prmStates, \prmInitState, \prmInputAlphabet, \prmOutputAlphabet,\allowbreak \prmTransitions, \prmOutput, \prmTerminals)$ be a PRM,
    and let $V \subset \prmStates$ be a set of $\mdp{M}$-unreachable PRM states of $\prm{A}$.
    Let $\prm{A'} = \prm{A} / V \rightarrow \alpha = (\prmStates, \prmInitState, \prmInputAlphabet, \prmOutputAlphabet \cup \mathset{\alpha}, \prmTransitions, \prmOutput^{\prm{A'}}, \prmTerminals)$ be a PRM obtained by setting the output of every transition into an unreachable state $\prmCommonState \in V$ to $\alpha$.
    In other words, $\prmOutput^{\prm{A'}}(\prmCommonState, \prmCommonState') = \alpha$ for all $\prmCommonState \in \prmStates$ and $\prmCommonState \in V$.
    Let $\pi^\star(\mdpCommonState, \prmCommonState)$ be an optimal policy in the product MDP $\mdp{M} \times \prm{A'}$.
    Then $\pi^\star(\mdpCommonState, \prmCommonState)$ is also an optimal policy in the product MDP $\mdp{M} \times \prm{A}$.
    \label{lemma:t2}
\end{mylemma}

\begin{proof}
    MDPs $\mdp{M} \times \prm{A}$ and $\mdp{M} \times \prm{A'}$ share the same state space $\mdpStates \times \prmStates$,
    probabilistic transition function $\mdpDynamics$,
    and initial state distribution.
    They may differ only in their (Markovian) reward function,
    specifically, on transitions into unreachable states
    $(\mdpCommonState, \prmCommonState) \in \mdpStates \times \prmStates$ for all $\mdpCommonState \in \mdpStates$ and $\prmCommonState \in V \subset \prmStates$.
    By definition~\ref{def:unreachable-prm-state},
    trajectories that induce a PRM transition into an $\mdp{M}$-unreachable state in $\prm{A}$ are unattainable.

    This proof proceeds similarly to proof of Lemma~\ref{lemma:t1},
    except it is even easier because we can work with a system of linear (not optimality) equations.
    The statement of the lemma concerning optimality will follow from the general reduction, i.e. the lemma holds for an arbitrary policy $\pi$ not just the optimal one.
    In Equation~\ref{eqn:bellman-t2} we set out the system of Bellman equations in $\mdp{M} \times \prm{A'}$, and show that it reduces to the one for $\mdp{M} \times \prm{A}$.

    \begin{equation}
        \begin{aligned}
            \qFunction^{\mdp{M} \times \prm{A'}}_{\policy} & (\mdpCommonState, \prmCommonState, \mdpCommonAction)
            =                                                                                                            \\
                                                           & =\sum_{\substack{\mdpCommonState' \in \mdpStates            \\ \prmCommonState' \in \prmStates}}
            \mdpDynamics(\mdpCommonState', \prmCommonState' \mid \mdpCommonState, \prmCommonState, \mdpCommonAction)
            \left(\prmOutput^{\prm{A'}}(\prmCommonState, \propInput,\prmCommonState') + \gamma \sum_{\mdpCommonAction' \in \mdpActions}
            \policy(\mdpCommonAction' \mid \mdpCommonState')
            \qFunction^{\mdp{M} \times \prm{A'}}_{\policy}(\mdpCommonState', \prmCommonState', \mdpCommonAction')\right) \\
                                                           & = \sum_{\substack{\mdpCommonState' \in \mdpStates           \\ \prmCommonState' \in \prmStates \setminus V}}
            \mdpDynamics(\mdpCommonState', \prmCommonState' \mid \mdpCommonState, \prmCommonState, \mdpCommonAction)
            \left(\prmOutput^{\prm{A}}(\prmCommonState, \propInput,\prmCommonState') + \gamma \sum_{\mdpCommonAction' \in \mdpActions}
            \policy(\mdpCommonAction' \mid \mdpCommonState')
            \qFunction^{\mdp{M} \times \prm{A'}}_{\policy}(\mdpCommonState', \prmCommonState', \mdpCommonAction')\right) \\
                                                           & + \sum_{\substack{\mdpCommonState' \in \mdpStates           \\ \prmCommonState' \in V}}
            \mdpDynamics(\mdpCommonState', \prmCommonState' \mid \mdpCommonState, \prmCommonState, \mdpCommonAction)
            \left(\prmOutput^{\prm{A'}}(\prmCommonState, \propInput,\prmCommonState') + \gamma \sum_{\mdpCommonAction' \in \mdpActions}
            \policy(\mdpCommonAction' \mid \mdpCommonState')
            \qFunction^{\mdp{M} \times \prm{A'}}_{\policy}(\mdpCommonState', \prmCommonState', \mdpCommonAction')\right)
        \end{aligned}
        \label{eqn:bellman-t2}
    \end{equation}

    If $\prmCommonState$ is an $\mdp{M}$-reachable state, the second sum vanishes because $\mdpDynamics^{\mdp{M} \times \prm{A}}(\mdpCommonState', \prmCommonState' \mid \mdpCommonState, \prmCommonState, \mdpCommonAction) = \mdpDynamics^{\mdp{M}}(\mdpCommonState', \mid \mdpCommonState, \mdpCommonAction) \cdot
        \prmTransitions^{\prm{A}}(\prmCommonState, \labelingFunction(\mdpCommonState, \mdpCommonAction, \mdpCommonState'), \prmCommonState') = 0$ for
    $\prmCommonState' \in V$
    (otherwise, $\prmCommonState'$ would be $\mdp{M}$-reachable if $\prmCommonState$ was $\mdp{M}$ reachable).
    Therefore, solutions in rows corresponding to $\mdp{M}$-reachable states are independent of rows corresponding to $\mdp{M}$-unreachable states,
    and the equations are the same as in the system for
    $\qFunction^{\mdp{M} \times \prm{A}}$, where the second sum also vanishes.
    Therefore, as Q-learning explores all reachable states infinitely often, and the reachable states in both product MDPs are the same and share the same Q-function, Q-learning will find the same optimal policy in both MDPs.
    The values of Q-functions corresponding to unreachable states are of no consequence.

\end{proof}


Definition~\ref{def:dependent-unreachable} captures the structure of the set of $\mdp{M}$-unreachable states in a PRM $\prm{A}$ induced by rejecting sinks states of a causal DFA.
This property of a set of unreachable states $V$ models the deterministic transition function of a causal DFA.

\begin{definition}[Dependent Set of Unreachable PRM States]
    \label{def:dependent-unreachable}
    Let $A = (\prmStates, \prmInitState, \prmInputAlphabet, \prmOutputAlphabet, \prmTransitions,\allowbreak \prmOutput, \prmTerminals)$ be a PRM, and
    $V$ a subset of $\mdp{M}$-unreachable states in $A$.
    We say that $V$ is a dependent set of $\mdp{M}$-unreachable states in $\prm{A}$ if it is a set of $\mdp{M}$-unreachable states and the following property holds for all labels $\propInput \in \prmInputAlphabet$
    and states $\prmCommonState \in \prmStates^{\prm{A}}$:
    $(\exists \prmCommonState' \in V)\; \prmTransitions^\prm{A}(\prmCommonState, \propInput, \prmCommonState') > 0 \implies (\forall \prmCommonState'' \not\in V)\; \prmTransitions^\prm{A}(\prmCommonState, \propInput, \prmCommonState'') = 0$.
\end{definition}

Intuitively, a set of $\mdp{M}$-unreachable states $V$ is dependent if it is not possible to transition into both $V$ and $U \setminus V$ from any $\prmCommonState \in \prmStates$.

\begin{mylemma}[Transformation 3]
    \label{lemma:t3}
    Let $A = (\prmStates, \prmInitState, \prmInputAlphabet, \prmOutputAlphabet, \prmTransitions, \prmOutput, \prmTerminals)$ be a PRM,
    $V$ a dependent set of $\mdp{M}$-unreachable states in $A$,
    and $m = -1 - \max_{r \in \prmOutputAlphabet^\prm{A}}{|r|} - \max_{\prmCommonState \in \prmStates^\prm{A}}{\optimalStateValueFunction(\prmCommonState)}$.
    Let $\prm{B_1} = A / V \rightarrow m$ be a PRM
    that mirrors the output of $A$,
    (except on transitions into unreachable states in $V$ where the output is $m$),
    and $\prm{B_2} = (-A) / V \rightarrow m$
    a PRM that negates the output of $A$
    (except on transitions into unreachable states in $V$ where the output is $m$).
    Let $\prmCommonState^0 \in \prmStates$ be a state in $\prm{A}$
    such that $\optimalStateValueFunction_{\prm{B_1}}(\prmCommonState) = \optimalStateValueFunction_{\prm{B_2}}(\prmCommonState) = 0$.
    Let $\prm{B} = (\prmStates, \prmInitState, \prmInputAlphabet, \prmOutputAlphabet, \prmTransitions, \prmOutput, \prmTerminals \cup \mathset{\prmCommonState^0})$ be a PRM obtained by adding $\prmCommonState^0$ to the set of terminal states in $\prm{A}$.
    Let $\pi^\star(\mdpCommonState, \prmCommonState)$ be an optimal policy in the product MDP $\mdp{M} \times \prm{B}$.
    Then $\pi^\star(\mdpCommonState, \prmCommonState)$ is also an optimal policy in the product MDP $\mdp{M} \times \prm{A}$.
\end{mylemma}

\begin{proof}
    Let $\pi$ be a policy in $\mdp{M} \times \prm{B}$.
    We will show that $\stateValueFunction^{\mdp{M} \times \prm{A}}_{\pi} = \stateValueFunction^{\mdp{M} \times \prm{B}}_{\pi}$,
    i.e. that an arbitrary policy $\pi$ has the same value in $\mdp{M} \times \prm{A}$.
    From there, it follows that if $\pi$ is optimal in $\mdp{M} \times \prm{B}$,
    then it is optimal in $\mdp{M} \times \prm{A}$.

    To make analysis easier, we will model terminal states as absorbing states (sinks with output $0$).
    For easier notation, $\tilde{\mdpCommonState}$ will refer to states in $\mdpStates^{\mdp{M}}$, and $\mdpCommonState = (\tilde{\mdpCommonState}, \prmCommonState)$ will refer to states in $\mdpStates^{\mdp{M} \times \prm{A}}$.

    We will first show $\stateValueFunction^{\mdp{M} \times \prm{A}}_{\pi}(\mdpCommonState^0) = \stateValueFunction^{\mdp{M} \times \prm{B}}_{\pi}(\mdpCommonState^0)$
    for every state $\mdpCommonState^0 = (\tilde{\mdpCommonState}, u^0)$.
    We have $\stateValueFunction^{\mdp{M} \times \prm{B}}_{\pi}(\mdpCommonState^0) = 0$ because $\mdpCommonState^0$ is an absorbing state in $\mdp{M} \times \prm{B}$.
    We must show that $\stateValueFunction^{\mdp{M} \times \prm{A}}_{\pi}(\mdpCommonState^0) = 0$.
    Intuitively, this is the case because in $\mdp{M} \times \prm{A}$, the expected return when starting in $\mdpCommonState^0 = (\tilde{\mdpCommonState}, \prmCommonState^0)$ and following an arbitrary policy $\pi$ is bounded with $0$ from above and below.
    It is enough to show that
    $\stateValueFunction^{\mdp{M} \times \prm{B_1}}_{\pi}(\mdpCommonState^0) = 0$, because by the proof of Lemma~\ref{lemma:t2} we have
    $\stateValueFunction^{\mdp{M} \times \prm{A}}_{\pi}(\mdpCommonState^0) =
        \stateValueFunction^{\mdp{M} \times \prm{B_1}}_{\pi}(\mdpCommonState^0)$.
    We will show this via Equation~\ref{eqn:bounds}.

    \begin{equation}
        0 = -\optimalStateValueFunction_{\prm{B_2}}(\prmCommonState^0) \leq
        \stateValueFunction^{\mdp{M} \times \prm{B_1}}_{\pi}(\mdpCommonState^0) \leq
        \optimalStateValueFunction_{\prm{B_1}}(\prmCommonState^0) = 0
        \label{eqn:bounds}
    \end{equation}

    Equalities in Equation~\ref{eqn:bounds} hold by assumption.
    The second inequality holds trivially.
    The first inequality holds because the ``pessimistic'' machine $\prm{B_2}$ realizes the minimal value of every state
    (negated to obtain the discounted return in terms of $\prm{B_1}$).
    More precisely, for a given state $\mdpCommonState = (\tilde{\mdpCommonState}, \prmCommonState)$ we have

    \begin{equation}
        \begin{aligned}
            \optimalStateValueFunction_{\prm{B_2}}(\prmCommonState) & = \max_{\propInput \in \prmInputAlphabet}{\sum_{\prmCommonState' \in \prmStates} \prmTransitions^\prm{B_2}(\prmCommonState, \propInput, \prmCommonState') (\prmOutput^\prm{B_2}(\prmCommonState, \propInput, \prmCommonState') + \gamma \optimalStateValueFunction_{\prm{B_2}}(\prmCommonState'))} \\
                                                                    & = \text{($\star$)}                                                                                                                                                                                                                                                                                 \\
                                                                    & = \max_{\substack{\propInput \in \prmInputAlphabet                                                                                                                                                                                                                                                 \\
            \prmTransitions(\prmCommonState, \propInput, \prmCommonState') = 0                                                                                                                                                                                                                                                                                           \\
                    \forall \prmCommonState' \in V}}
            {\sum_{\prmCommonState' \in \prmStates} \prmTransitions^\prm{B_2}(\prmCommonState, \propInput, \prmCommonState') (\prmOutput^\prm{B_2}(\prmCommonState, \propInput, \prmCommonState') + \gamma \optimalStateValueFunction_{\prm{B_2}}(\prmCommonState'))}                                                                                                    \\
                                                                    & = \max_{\substack{\propInput \in \prmInputAlphabet                                                                                                                                                                                                                                                 \\
            \prmTransitions(\prmCommonState, \propInput, \prmCommonState') = 0                                                                                                                                                                                                                                                                                           \\
                    \forall \prmCommonState' \in V}}
            {\sum_{\prmCommonState' \in \prmStates} \prmTransitions^\prm{B_1}(\prmCommonState, \propInput, \prmCommonState') (-\prmOutput^\prm{B_1}(\prmCommonState, \propInput, \prmCommonState') + \gamma \optimalStateValueFunction_{\prm{B_2}}(\prmCommonState'))}
        \end{aligned}
        \label{eqn:value-iteration-1}
    \end{equation}

    and similarly

    \begin{equation}
        \begin{aligned}
            \optimalStateValueFunction_{\prm{B_1}}(\prmCommonState) & = \max_{\propInput \in \prmInputAlphabet}{\sum_{\prmCommonState' \in \prmStates} \prmTransitions^\prm{B_1}(\prmCommonState, \propInput, \prmCommonState') (\prmOutput^\prm{B_1}(\prmCommonState, \propInput, \prmCommonState') + \gamma \optimalStateValueFunction_{\prm{B_2}}(\prmCommonState'))} \\
                                                                    & = \max_{\substack{\propInput \in \prmInputAlphabet                                                                                                                                                                                                                                                 \\
            \prmTransitions(\prmCommonState, \propInput, \prmCommonState') = 0                                                                                                                                                                                                                                                                                           \\
                    \forall \prmCommonState' \in V}}
            {\sum_{\prmCommonState' \in \prmStates} \prmTransitions^\prm{B_1}(\prmCommonState, \propInput, \prmCommonState') (\prmOutput^\prm{B_1}(\prmCommonState, \propInput, \prmCommonState') + \gamma \optimalStateValueFunction_{\prm{B_1}}(\prmCommonState'))}
        \end{aligned}
        \label{eqn:value-iteration-2}
    \end{equation}.

    ($\star$): By assumption, when $\prmTransitions^\prm{B_2}(\prmCommonState, \propInput, \prmCommonState') > 0$ for any $\prmCommonState' \in V$ then $\prmTransitions^\prm{B_2}(\prmCommonState, \propInput, \prmCommonState'') = 0$ for all $\prmCommonState'' \not\in V$.
    Intuitively, if reading input $\propInput$ from state $\prmCommonState$ induces a transition to
    $\prmCommonState' \in V$ with positive probability in $\prm{B_2}$,
    since the transitions of the causal DFA are \emph{not} probabilistic,
    every other state $\prmCommonState''$ such that $\prmTransitions^\prm{B_2}(\prmCommonState, \propInput, \prmCommonState'') > 0$ must also transition into the same rejecting sink state in the causal DFA.
    In that case, $\prmOutput^\prm{B_2}(\prmCommonState, \propInput, \prmCommonState') = m$ for all $\prmCommonState'$, which is lower than any possible immediate reward and resulting state value.
    Therefore, the maximum is not attained for the input $\propInput$.
    Similar reasoning is applied in Equation~\ref{eqn:value-iteration-2}.

    Equation~\ref{eqn:value-iteration-1} and Equation~\ref{eqn:value-iteration-2} show that $-\optimalStateValueFunction_{\prm{B_2}}$ bounds the value of any policy in $\mdp{M} \times \prm{B_1}$ from below, that is $-\optimalStateValueFunction_{\prm{B_2}}(\prmCommonState) \leq \stateValueFunction^{\mdp{M} \times \prm{B_1}}_{\pi}(\mdpCommonState)$.
    The argument is that $\optimalStateValueFunction_{\prm{B_2}}$ is a solution to the Bellman optimality equation in $\prm{B_1}$ with negated rewards.
    In particular, Equation~\ref{eqn:value-iteration-1} and Equation~\ref{eqn:value-iteration-2} show that one can disregard transitions into unreachable states in $V$ when computing state values.
    In that case, $-\optimalStateValueFunction_{\prm{B_2}}$ is the pessimal state value in $\prm{B_1}$.
    Intuitively, one attains a discounted return of $-\optimalStateValueFunction_{\prm{B_2}}(\prmCommonState)$
    when starting in $\prmCommonState$ and minimizing the expected discounted sum of rewards along transitions in $\prm{B_1}$,
    while ignoring transitions into $V$.

    We proceed to show
    $\stateValueFunction^{\mdp{M} \times \prm{A}}_{\pi}(\mdpCommonState) = \stateValueFunction^{\mdp{M} \times \prm{B}}_{\pi}(\mdpCommonState)$ in all components (not just for $\mdpCommonState = \mdpCommonState^0$).

    When we fix an arbitrary policy $\pi$,
    we obtain the immediate reward vector
    $\mdpMatrixReward^{\prm{A};\pi}_\mdpCommonState$ ($\mdpMatrixReward^{\prm{B};\pi}_\mdpCommonState$)
    and probabilistic transition matrix
    $\mdpMatrixTransition^{\prm{A};\pi}_{\mdpCommonState, \mdpCommonState'}$ ($\mdpMatrixTransition^{\prm{B};\pi}_{\mdpCommonState, \mdpCommonState'}$)
    for $\mdp{M} \times \prm{A}$ ($\mdp{M} \times \prm{B}$).

    The state value Bellman equation for $\pi$ in $\mdp{M} \times \prm{A}$ can then be expressed in matrix form as in Equation~\ref{eqn:bellman-state-value}
    (and similarly for $\mdp{M} \times \prm{B}$).

    \begin{equation}
        \stateValueFunction^{\prm{A}}_{\pi} =
        \mdpMatrixReward^{\prm{A};\pi} +
        \gamma \mdpMatrixTransition^{\prm{A};\pi}
        \stateValueFunction^{\prm{A}}_{\pi}
        \label{eqn:bellman-state-value}
    \end{equation}

    We proceed to show that $\stateValueFunction^{\prm{B}}_{\pi}$ also solves Equation~\ref{eqn:bellman-state-value}, that is that Equation~\ref{eqn:bellman-state-value-mix} holds.

    \begin{equation}
        \stateValueFunction^{\prm{B}}_{\pi} =
        \mdpMatrixReward^{\prm{A};\pi} +
        \gamma \mdpMatrixTransition^{\prm{A};\pi}
        \stateValueFunction^{\prm{B}}_{\pi}
        \label{eqn:bellman-state-value-mix}
    \end{equation}

    We already know that the Equation~\ref{eqn:bellman-state-value-mix} holds in rows corresponding to $\mdpCommonState^0 \in \mathset{(\tilde{\mdpCommonState}, \prmCommonState^0) : \tilde{\mdpCommonState} \in \mdpStates^\mdp{M}}$,
    as we have shown $\stateValueFunction^{\prm{A}}_{\pi}(\mdpCommonState^0) = \stateValueFunction^{\prm{B}}_{\pi}(\mdpCommonState^0) = 0$.
    For $\mdpCommonState \in \mdpStates \setminus \mathset{(\tilde{\mdpCommonState}, \prmCommonState^0) : \tilde{\mdpCommonState} \in \mdpStates^\mdp{M}}$, we have $\mdpMatrixReward^{\prm{A};\pi}_{\mdpCommonState} = \mdpMatrixReward^{\prm{B};\pi}_{\mdpCommonState}$
    ($\prm{B}$ does not change the immediate reward on transitions from $\mdpCommonState \neq \mdpCommonState^0$).
    However, we also have $\mdpMatrixTransition^{\prm{A};\pi}_{\mdpCommonState, :} = \mdpMatrixTransition^{\prm{B};\pi}_{\mdpCommonState, :}$
    ($\mdp{M} \times \prm{B}$ transitions in the same way as $\mdp{M} \times \prm{A}$ from $\mdpCommonState \neq \mdpCommonState^0$).

\end{proof}


Now we can prove Theorem~\ref{thm:convergence}.

\begin{proof}
    Algorithm~\ref{alg:causal} starts with a PRM $\prm{A}$
    and applies a series transformations in order to obtain a new PRM $\prm{B}$.
    Then, it runs QRM for $(\mdp{M}, \mdp{B})$ instead of $(\mdp{M}, \mdp{A})$.
    Lemmas~\ref{lemma:t1},~\ref{lemma:t2}, and~\ref{lemma:t3} show that the optimal policy either
    remains the same when the transformations are applied (Transformation 2, 3),
    or that the optimal policy for the initial PRM can be easily recovered from the transformed PRM (Transformation 1).
    Line~\ref{line:product} applies Transformation 1 and 2.
    Lines~\ref{line:value-iteration}-\ref{line:t3-end} apply Transformation 3.
    Finally, convergence to optimal policy of Algorithm~\ref{alg:causal} then follows from the convergence to optimal policy of QRM.

\end{proof}

%% file: appendix_sections/3_additional_results.tex
\section{Four Doors Case Study}
\label{sec:4-door-case-study}

Figure~\ref{fig:prm-4-doors} depicts the PRM for the 4-door task. As can be seen, increasing the number of doors leads to an exponential increase in the number of states.

\input{figures/prm-4-doors}

%% file: figures/prm-4-doors.tex
\begin{figure}[t]
  \vspace{-5cm}
  \centering
  \hspace*{-6cm}
  \begin{tikzpicture}[->,>=stealth',shorten >=1pt,auto,node distance=3cm,semithick, scale=0.55]
    \node[state, initial, initial text=] (q0) at (-6, 0) {$q_0$};
    \node[state] (q1) at (-2, 10) {$q_1$};
    \node[state] (q2) at (1, -10) {$q_2$};
    \node[state] (q3) at (1, 3.5) {$q_3$};
    \node[state] (q4) at (1, -1.5) {$q_4$};

    \node[state] (q7) at (7, 8.5) {$q_7$};
    \node[state] (q5) at (7, 12) {$q_5$};
    \node[state] (q6) at (7, 4) {$q_6$};

    \node[state] (q10) at (7, -1.5) {$q_{10}$};
    \node[state] (q8) at (10, -3) {$q_{8}$};
    \node[state] (q9) at (10, -8) {$q_{9}$};

    \node[state] (q11) at (16,8) {$q_{11}$};
    \node[state] (q12) at (13, 5) {$q_{12}$};
    \node[state] (q13) at (13, 0) {$q_{13}$};
    \node[state] (q14) at (15, -10) {$q_{14}$};

    \node[state, accepting] (q15) at (17, 0) {$q_{15}$};

    \path (q0) edge[loop below] node[below,xshift = -2ex ]{$\lnot(\texttt{a} \lor \texttt{b} \lor \texttt{c} \lor \texttt{d} )$, $0$}node[below, yshift=-3ex]{\highlightyellow{$1$}}(q0);
    \path (q0) edge[] node[sloped]{$\texttt{a} \land \lnot(\texttt{b} \lor \texttt{c} \lor \texttt{d})$, $0$} node[sloped, swap]{\highlightyellow{$1$}} (q1);
    \path (q0) edge[] node[sloped, pos =0.3]{$\texttt{c} \lnot \texttt{d}$, $0$} node[sloped, swap, pos = 0.3]{\highlightyellow{$1$}} (q3);
    \path (q0) edge[bend left=10, looseness=1.5] node[midway,above, sloped]{$\texttt{b} \land \lnot (\texttt{c} \lor \texttt{d})$, $0$} node[sloped, swap]{\highlightyellow{$1$}} (q2);
    \path (q0) edge[] node[sloped]{$ \texttt{d}$, $0$} node[sloped, swap]{\highlightyellow{$1$}} (q4);

    \path (q1) edge[loop left] node[above, yshift=1ex]{$\lnot(\texttt{b} \lor \texttt{c} \lor \texttt{d})$, $0$}node[below, yshift=2.5ex]{\highlightyellow{$1$}}(q1);
    \path (q1) edge[bend left=20, looseness=1.5] node[midway,above, sloped]{$\texttt{b} \land \lnot(\texttt{c} \lor \texttt{d})$, $0$} node[sloped, swap]{\highlightyellow{$1$}} (q5);
    \path (q1) edge[bend left=10, looseness=1.5] node[midway,above, sloped]{$\texttt{c} \land \lnot \texttt{d}$, $0$} node[sloped, swap]{\highlightyellow{$1$}} (q6);
    \path (q1) edge[bend left=20, looseness=1.5] node[midway,above, sloped]{$ \texttt{d}$, $0$} node[sloped, swap]{\highlightyellow{$1$}} (q7);

    \path (q2) edge[loop below] node[below]{$\lnot(\texttt{a} \lor \texttt{c} \lor \texttt{d})$, $0$}node[below, yshift=-3ex]{\highlightyellow{$1$}}(q2);
    \path (q2) edge[bend left=120, looseness=2] node[midway,above, sloped]{$\texttt{a} \lnot(\texttt{c} \lor \texttt{d})$, $0$}node[sloped, swap]{\highlightyellow{$0.9$}}(q5);
    \path (q2) edge[bend right=10, looseness=1.5] node[midway,above, sloped]{$\texttt{a} \lnot(\texttt{c} \lor \texttt{d})$, $0$}node[sloped, swap]{\highlightyellow{$0.1$}}(q4);
    \path (q2) edge[bend right=10, looseness=1] node[midway,above, sloped]{$\texttt{c} \lnot\texttt{d}$, $0$}node[sloped, swap]{\highlightyellow{$1$}}(q8);
    \path (q2) edge[] node[sloped]{$\texttt{d}$, $0$}node[sloped, swap]{\highlightyellow{$1$}}(q9);

    \path (q3) edge[loop above] node[above, yshift=1ex]{$\lnot(\texttt{a} \lor \texttt{b} \lor \texttt{d})$, $0$}node[above, yshift=-1ex]{\highlightyellow{$1$}}(q3);
    \path (q3) edge[] node[sloped]{$\texttt{a} \lnot(\texttt{(b}\lor\texttt{d})$, $0$}node[sloped, swap]{\highlightyellow{$1$}}(q6);
    \path (q3) edge[bend left=20, looseness=0] node[midway,above, sloped]{$\texttt{b} \lnot\texttt{d}$, $0$}node[sloped, swap]{\highlightyellow{$1$}}(q8);
    \path (q3) edge[bend right=10, looseness=1.5] node[midway,above, sloped]{$\texttt{d}$, $0$}node[sloped, swap]{\highlightyellow{$1$}}(q10);

    \path (q4) edge[loop above] node[above, yshift=1ex]{$\lnot(\texttt{a} \lor \texttt{b} \lor \texttt{c})$, $0$}node[above, yshift=-1ex]{\highlightyellow{$1$}}(q4);
    \path (q4) edge[bend left=80, looseness=1.75] node[midway,above, sloped]{$\texttt{a} \lnot(\texttt{(b}\lor\texttt{c})$, $0$}node[sloped, swap]{\highlightyellow{$1$}}(q7);
    \path (q4) edge[] node[sloped, pos = 0.3]{$\texttt{b} \lnot\texttt{c}$, $0$}node[sloped, swap, pos = 0.3]{\highlightyellow{$1$}}(q9);
    \path (q4) edge[] node[sloped]{$\texttt{c}$, $0$}node[sloped, swap]{\highlightyellow{$1$}}(q10);

    \path (q5) edge[loop above] node[above,yshift =2ex]{$\lnot(\texttt{c} \lor \texttt{d})$, $0$}node[above, yshift=0ex]{\highlightyellow{$1$}}(q5);
    \path (q5) edge[bend left=20, looseness=1.5] node[midway,above, sloped,pos = 0.35]{$\texttt{c} \lnot\texttt{d}$, $0$}node[sloped, swap, pos = 0.35]{\highlightyellow{$1$}}(q11);
    \path (q5) edge[bend left=25, looseness=1.75] node[midway,above, sloped]{$\texttt{d}$, $0$}node[sloped, swap]{\highlightyellow{$1$}}(q12);

    \path (q6) edge[loop below] node[below,yshift = 1ex]{$\lnot(\texttt{b} \lor \texttt{d})$, $0$}node[below, yshift=-1ex]{\highlightyellow{$1$}}(q6);
    \path (q6) edge[] node[sloped, pos=0.37]{$\texttt{b} \lnot\texttt{d}$, $0$}node[sloped, swap, pos=0.4]{\highlightyellow{$1$}}(q11);
    \path (q6) edge[] node[sloped]{$\texttt{d}$, $0$}node[sloped, swap]{\highlightyellow{$1$}}(q13);

    \path (q7) edge[loop below] node[below]{$\lnot(\texttt{b} \lor \texttt{c})$, $0$}node[below, yshift=-3ex]{\highlightyellow{$1$}}(q7);
    \path (q7) edge[bend left=10, looseness=1.5] node[midway,above, sloped]{$\texttt{b} \lnot\texttt{c}$, $0$}node[sloped, swap]{\highlightyellow{$1$}}(q12);
    \path (q7) edge[] node[sloped]{$\texttt{c}$, $0$}node[sloped, swap]{\highlightyellow{$1$}}(q13);

    \path (q8) edge[loop right, in=60, out=10, looseness =3 ] node[above,sloped, xshift=-1ex, yshift=0ex]{$\lnot(\texttt{a} \lor \texttt{d})$, $0$} node[above, sloped, xshift=0ex, yshift=-4ex]{\highlightyellow{$1$}}(q8);

    \path (q8) edge[bend right=70, looseness=2.5] node[midway,above, sloped, pos = 0.3]{$\texttt{a} \lnot\texttt{d}$, $0$}node[sloped, swap, pos = 0.3]{\highlightyellow{$1$}}(q11);
    \path (q8) edge[bend left=10, looseness=1.5] node[midway,above, sloped, pos = 0.85]{$\texttt{d}$, $0$}node[sloped, swap, pos = 0.8]{\highlightyellow{$1$}}(q12);

    \path (q9) edge[loop below] node[below]{$\lnot(\texttt{a} \lor \texttt{c})$, $0$}node[below, yshift=-3ex]{\highlightyellow{$1$}}(q9);
    \path (q9) edge[bend right=90, looseness=2] node[midway,above, sloped, pos = 0.35]{$\texttt{a} \lnot\texttt{c}$, $0$}node[sloped, swap, pos = 0.35]{\highlightyellow{$1$}}(q12);
    \path (q9) edge[bend right=10, looseness=1.5] node[midway,above, sloped]{$\texttt{c}$, $0$}node[sloped, swap]{\highlightyellow{$1$}}(q14);

    \path (q10) edge[loop below] node[below]{$\lnot(\texttt{a} \lor \texttt{b})$, $0$}node[below, yshift=-3ex]{\highlightyellow{$1$}}(q10);
    \path (q10) edge[] node[sloped, pos = 0.4]{$\texttt{a} \lnot\texttt{b}$, $0$}node[sloped, swap, pos = 0.4]{\highlightyellow{$1$}}(q13);
    \path (q10) edge[] node[sloped]{$\texttt{b}$, $0$}node[sloped, swap]{\highlightyellow{$1$}}(q14);

    \path (q11) edge[loop right] node[right]{$\lnot\texttt{d} $, $0$}node[right, yshift=-3ex]{\highlightyellow{$1$}}(q11);
    \path (q11) edge[] node[sloped, pos = 0.3]{$\texttt{d}$, $0$}node[sloped, swap, pos = 0.3]{\highlightyellow{$1$}}(q15);

    \path (q12) edge[loop below] node[below]{$\lnot\texttt{c} $, $0$}node[below, yshift=-3ex]{\highlightyellow{$1$}}(q12);
    \path (q12) edge[] node[sloped]{$\texttt{c}$, $0$}node[sloped, swap]{\highlightyellow{$1$}}(q15);

    \path (q13) edge[loop below] node[below]{$\lnot\texttt{b} $, $0$}node[below, yshift=-3ex]{\highlightyellow{$1$}}(q13);
    \path (q13) edge[] node[sloped]{$\texttt{b}$, $0$}node[sloped, swap]{\highlightyellow{$1$}}(q15);

    \path (q14) edge[loop below] node[below]{$\lnot\texttt{a} $, $0$}node[below, yshift=-3ex]{\highlightyellow{$1$}}(q14);
    \path (q14) edge[] node[sloped]{$\texttt{a}$, $0$}node[sloped, swap]{\highlightyellow{$1$}}(q15);

    \path (q15) edge[loop right] node[above, yshift=1ex]{$true$}node[below]{\highlightyellow{$ $}}(q15);

  \end{tikzpicture}
  \caption{The PRM without causal info about the four-door task}
  \label{fig:prm-4-doors}
\end{figure}
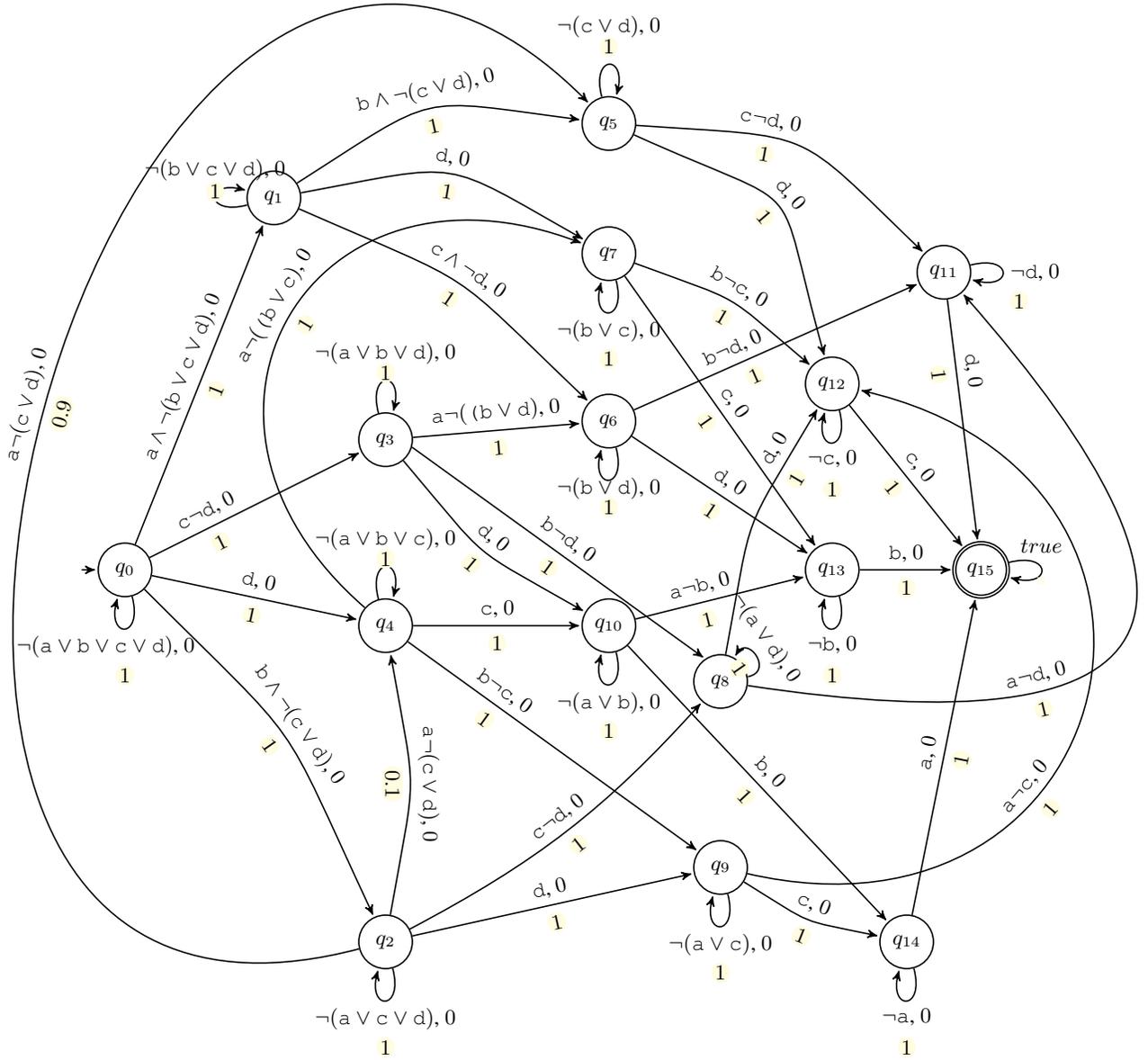